%% file: gaming.tex
\documentclass[letterpaper,11pt]{article}

\include{mary_macros_sane}

\usepackage{amsmath}
\usepackage{amsfonts}
\usepackage{amssymb}
\usepackage{amsthm}
\usepackage[usenames]{color}
\usepackage{fullpage}
\usepackage{xspace}
\usepackage{url,ifthen}
\usepackage{srcltx}
\usepackage{multirow}
\usepackage{boxedminipage}
\usepackage[margin=1.125in]{geometry}
\usepackage{nicefrac}
\usepackage{xspace}
\usepackage{color}
\definecolor{DarkGreen}{rgb}{0.1,0.5,0.1}
\definecolor{DarkRed}{rgb}{0.5,0.1,0.1}
\definecolor{DarkBlue}{rgb}{0.1,0.1,0.5}

\usepackage[small]{caption}

\usepackage{tikz}
\def\draft{1}

\def\submit{0}

\ifnum\draft=1 
    \def\ShowAuthNotes{1}
\else
    \def\ShowAuthNotes{0}
\fi

\ifnum\submit=1
\newcommand{\forsubmit}[1]{#1}
\newcommand{\forreals}[1]{}
\else
\newcommand{\forreals}[1]{#1}
\newcommand{\forsubmit}[1]{}
\fi

\ifnum\ShowAuthNotes=1
\newcommand{\authnote}[2]{{ \footnotesize \bf{\color{DarkRed}[#1's Note:
{\color{DarkBlue}#2}]}}}
\else
\newcommand{\authnote}[2]{}
\fi

\newtheorem{theorem}{Theorem}[section]
\newtheorem{remark}[theorem]{Remark}

\newtheorem{corollary}[theorem]{Corollary}
\newtheorem{proposition}[theorem]{Proposition}
\newtheorem{claim}[theorem]{Claim}

\theoremstyle{definition}
\newtheorem{definition}[theorem]{Definition}

\newcommand{\chapterref}[1]{\hyperref[ch:#1]{Chapter~\ref{ch:#1}}}

\newcommand{\claimref}[1]{\hyperref[claim:#1]{Claim~\ref{claim:#1}}}

\newcommand{\corollaryref}[1]{\hyperref[cor:#1]{Corollary~\ref{cor:#1}}}

\newcommand{\definitionref}[1]{\hyperref[def:#1]{Definition~\ref{def:#1}}}

\newcommand{\equationref}[1]{\hyperref[eq:#1]{Equation~\ref{eq:#1}}}

\newcommand{\factref}[1]{\hyperref[fact:#1]{Fact~\ref{fact:#1}}}

\newcommand{\figureref}[1]{\hyperref[fig:#1]{Figure~\ref{fig:#1}}}

\newcommand{\itemref}[1]{\hyperref[item:#1]{Item~(\ref{item:#1})}}

\newcommand{\lemmaref}[1]{\hyperref[lem:#1]{Lemma~\ref{lem:#1}}}

\newcommand{\propref}[1]{\hyperref[prop:#1]{Proposition~\ref{prop:#1}}}

\newcommand{\propositionref}[1]{\hyperref[prop:#1]{Proposition~\ref{prop:#1}}}

\newcommand{\remarkref}[1]{\hyperref[rem:#1]{Remark~\ref{rem:#1}}}
\newcommand{\sectionlabel}[1]{\label{sec:#1}}
\newcommand{\sectionref}[1]{\hyperref[sec:#1]{Section~\ref{sec:#1}}}

\newcommand{\theoremref}[1]{\hyperref[thm:#1]{Theorem~\ref{thm:#1}}}

\usepackage[T1]{fontenc}
\usepackage{kpfonts}
\usepackage{microtype}

\newcommand{\Esymb}{\mathbb{E}}
\newcommand{\Psymb}{\mathbb{P}}

\DeclareMathOperator*{\E}{\Esymb}

\DeclareMathOperator*{\ProbOp}{\Psymb r}
\renewcommand{\Pr}{\ProbOp}

\newcommand{\mper}{\,.}

\newcommand{\cA}{{\cal A}}

\newcommand{\cC}{{\cal C}}
\newcommand{\cD}{{\cal D}}

\newcommand{\cH}{{\cal H}}

\newcommand{\cS}{{\cal S}}

\renewcommand{\leq}{\leqslant}
\renewcommand{\le}{\leqslant}
\renewcommand{\geq}{\geqslant}
\renewcommand{\ge}{\geqslant}

\newcommand{\Set}[1]{\left\{#1\right\}}

\newcommand{\signs}{\{-1,1\}}

\newcommand{\R}{\mathbb{R}}

\usepackage{bm}
\renewcommand{\vec}[1]{{\bm{#1}}}

\newcommand{\poly}{{\rm poly}}

\renewcommand{\epsilon}{\varepsilon}

\newcommand{\eps}{\epsilon}

\newcommand{\remove}[1]{}

\newcommand{\jr}{\textsc{Jury}\xspace}
\newcommand{\cnt}{\textsc{Contestant}\xspace}
\newcommand{\ct}{\textsc{Contestant}\xspace}
\usepackage{enumitem}
{\end{itemize}}

{\end{enumerate}}

\title{Strategic Classification}
\author{Moritz Hardt\and Nimrod Megiddo\and Christos Papadimitriou\and Mary
Wootters}

\begin{document}
\maketitle
\begin{abstract}

\input{abstract}

\end{abstract}

\vfill
\thispagestyle{empty}
\pagebreak

%
%

\sloppy
\section{Introduction}
\input{intro}

\section{Separable cost functions}
\input{separable}

\section{General cost functions}

\input{general}

\section{NP-completeness}
\label{sec:npcomplete}
\input{np_complete}

\section{Experiments}
\label{sec:experiments}
\input{experiments}

\section*{Acknowledgments}

We are grateful for stimulating discussions with Cynthia Dwork, Brendan Juba, Silvio Micali, Omer
Reingold and Aaron Roth. We are also grateful to Fabricio Benevenuto for pointing us to the Apondator data set and sharing it with us,
and to an anonymous reviewer for pointing out that our uniform result implied the non-uniform corollary.

\bibliographystyle{alpha}
\bibliography{refs}

\appendix

\remove{
\section{Alternative solution concepts}
\sectionlabel{alternatives}
\input{alternatives}
}

\section{Proof of Theorem \ref{lem:minsep}}\label{sec:minproof}
\input{minproof}
\end{document}

%% file: mary_macros_sane.tex
\usepackage[ruled,linesnumbered]{algorithm2e}

\newcommand{\thresh}[2]{#1^{(\geq #2)}}

\newcommand{\A}{\ensuremath{\mathcal{A}}}

\newcommand{\PR}[1]{{\mathbb{P}}\left\{ #1\right\}}
\newcommand{\EE}{\mathbb{E}}

\newcommand{\twonorm}[1]{\left\|#1\right\|_2}

\newcommand{\inabs}[1]{\left|#1\right|}
\newcommand{\ip}[2]{\ensuremath{\left\langle #1,#2\right\rangle}}

\newcommand{\inset}[1]{\left\{#1\right\}}
\newcommand{\inbrac}[1]{\left\{#1\right\}}
\newcommand{\inparen}[1]{\left(#1\right)}
\newcommand{\inbrak}[1]{\left[#1\right]}
\newcommand{\suchthat}{\,:\,}

\newcommand{\argmax}{\mathrm{argmax}}
\newcommand{\argmin}{\mathrm{argmin}}

\newcommand{\ind}[1]{\ensuremath{\mathbf{1}\Set{#1}}}

\renewcommand{\epsilon}{\varepsilon}

\newcommand{\mkw}[1]{{\footnotesize  [\textbf{\textcolor{red}{#1}} --mary]\normalsize}}
\newcommand{\TODO}[1]{\textcolor{red}{\textbf{TODO: #1}}}

%% file: abstract.tex
Machine learning relies on the assumption that unseen test instances of a classification problem follow the same distribution as observed training data.  However, this principle can break down when machine learning is used to make important decisions about the welfare (employment, education, health) of strategic individuals. Knowing information about the classifier, such individuals may manipulate their attributes in order to obtain a better classification outcome. As a result of this behavior---often referred to as \emph{gaming}---the performance of the classifier may deteriorate sharply. Indeed, gaming is a well-known obstacle for using machine learning methods in practice; in financial policy-making, the problem is widely known as Goodhart's law.  In this paper, we formalize the problem, and pursue algorithms for learning classifiers that are robust to gaming.  

We model classification as a sequential game between a player named ``Jury'' and a player named ``Contestant.'' Jury designs a classifier, and Contestant receives an input to the classifier drawn from a distribution.  Before being classified, Contestant may change his input based on Jury's classifier. However, Contestant incurs a cost for these changes according to a cost function. Jury's goal is to achieve high classification accuracy with respect to Contestant's original input and some underlying target classification function, assuming Contestant plays best response. Contestant's goal is to achieve a favorable classification outcome while taking into account the cost of achieving it.  

For a natural class of {\em separable} cost functions, and certain generalizations, we obtain computationally efficient learning algorithms which are near optimal, achieving a classification error that is arbitrarily close to the theoretical minimum.  Surprisingly, 
our algorithms are efficient even on concept classes that are computationally hard to learn.  For general cost functions, designing an approximately optimal strategy-proof classifier, for inverse-polynomial approximation, is NP-hard.

%% file: intro.tex

\renewcommand{\jr}{{Jury}\xspace}
\renewcommand{\cnt}{{Contestant}\xspace}
\renewcommand{\ct}{{Contestant}\xspace}
Studies have found that a student's success at school is highly correlated with 
the \emph{number of books in the parents' household}~\cite{Evans}.  
Therefore, in theory, this attribute should be of great value 
when using machine-learning techniques for student admission.
However, this statistical pattern is obviously open to manipulation: books are relatively cheap and, knowing that their number matters, parents can easily buy an attic full of unread books 
in preparation for admission decisions.

This behavior is often called \emph{gaming}: the strategic use of methods that, 
while not dishonest or against the rules, give the individual an unintended
advantage.\footnote{See, for instance, \url{http://www.thefreedictionary.com/gamesmanship}.}
The problem of gaming is well known and can be seen as a consequence of 
a classical principle in financial policy making known as \emph{Goodhart's law}:
\begin{quote}
\emph{``If a measure becomes the public's goal, it is no longer a good measure.''}
\end{quote}
Goodhart's law is highly relevant for the practice of machine learning
today.
Machine learning relies on the idea that patterns observed
in training data translate to accurate predictions about unseen
instances of a classification problem. 
Machine learning is increasingly used to make decisions about
individuals in areas such as employment, health, education and
commerce. 
In each such application, an individual may try to achieve a more 
favorable classification outcome with little effort by
exploiting information that may be available about the classifier. 
Goodhart's law suggests that if a classifier is exposed to public scrutiny, its
prediction accuracy vanishes and it becomes useless. 
Indeed, concerns of gaming and manipulation are often used as a reason for keeping classification
mechanisms secret, which is a major concern in credit scoring (cf.~\cite{CitronP14}).
Secrecy is not a robust solution to the problem; information about a
classifier may leak, and it is often possible for an outsider to learn
such information from classification outcomes. 
Moreover, transparency is highly desirable and
sometimes even mandated by regulation in applications of public interest.


Our goal in this work is to formalize gaming in classification and to 
develop approaches and techniques for designing
classifiers that are near optimal in the presence of public scrutiny
and gaming.  
The hope is that this analysis may lead, in certain cases, 
to classifiers with performance comparable to ones that rely on secrecy.
In other cases, our analysis may lead to the realization that secrecy is
necessary for a good classification performance.


As gaming entails strategic behavior, any attempt to formalize
it must incorporate the strategic response of an individual to a classifier.
We propose a general model for \emph{strategic classification}, based on a
sequential two-player game between a party that wishes to learn a classifier
and a party that is being classified.  
This is different from the
standard supervised-learning setup, which is commonly viewed as a one-shot
learning process, in which an algorithm produces a classifier from labeled
training examples. 
Our model combines the statistical elements of learning
theory---namely, seeking a small generalization error given a finite number of
training data---with a game-theoretic notion of equilibrium. 
This combination allows us to
build classifiers that achieve high classification accuracy at equilibrium,
when both parties respond strategically to each other.

\paragraph{Informal description of our model and results.}
We model learning and classification as a sequential two-player game.  
The first player, named ``\jr," has a learning task: she is given labeled examples from some true classifier $h$, and must publish a classifier $f$.
The second player, named ``\ct," receives an input to the classifier, and is given a chance to ``game" it.  That is, 
\ct may change his input based on $f$.
However, \ct incurs a cost for these changes according to a \em cost function \em known to both
players. \jr's goal is to achieve high classification accuracy with
respect to 
\ct's \em original \em input and the true classifier $h$.
\ct's goal is to be accepted by \jr, without paying too much to change his input.
The cost function plays an important role in our framework as it
determines the flexibility of \ct in changing his input. 
Ideally, the cost function should capture ground truth or our best approximation thereof. 

Our contributions are the following:
\begin{itemize}
\item 
For certain cost functions, we give an efficient strategy for \jr which approaches the optimal payoff. 
Surprisingly, this result holds even for concept classes which are computationally intractable to learn.
The intuitive reason is that \ct's changes to his input
``smooth out'' any intractability.
\item  Those cost functions for which \jr has near-optimal algorithms include {\em separable cost functions}.  This is a natural class of cost functions which generalize our introductory example of
school admissions and books.  
We also obtain results for a broad generalization of these separable functions.
\item 
In contrast, we show that, for general cost functions---even for cost functions which are metrics, another nice class---it is hard to approximate the optimum classification score with reasonable accuracy.
\item 
We observe through experiments on real data that our approach leads to higher
classification accuracy compared with standard classifiers in situations where
even a small amount of gaming occurs. We also experimentally demonstrate the
robustness of our framework to inaccuracies in our modeling assumptions and
the modeling of the cost function. 
\end{itemize}

\subsection{Our model}

We first describe an idealized version of the game, where \jr has perfect information. 
This will serve as a reference point for how well \jr may hope to do.  We will later relax this to a version
where \jr knows neither $h$ nor $\cD$, and only sees labeled examples.
\begin{definition}[Full information game]\label{game:seq}
The players are \jr and \ct.  
Fix a \em population \em $X$, and a probability distribution
$\cD$ over $X$.  
Fix a \em cost function \em $c: X \times X \to \R_+$ and a
\emph{target} classifier $h:X \to \signs$.
\begin{enumerate}
\item 
 \jr (who knows the cost function $c$, the distribution $\cD$, and
the true classifier $h$) publishes a classifier $f: X \to \signs$.  
\item 
 \ct (who knows $c,h,\cD$, and $f$), produces a 
function $\Delta: X \to X$.
\end{enumerate}
The payoff to \jr is 
$\Pr_{x \sim \cD} \inbrac{ h(x) = f(\Delta(x)) }$.
The payoff to \ct is
$\E_{x \sim \cD} \inbrak{f(\Delta(x)) - c(x, \Delta(x))}$.
\end{definition}

Definition~\ref{game:seq} is an example of a {\em Stackelberg competition},
which means that the first player (\jr) has the ability to commit
to her strategy (a classifier $f$) before the second player (\ct) responds. 
We wish to find a {\em Stackelberg equilibrium}, that is,
a highest-payoff strategy for \jr, assuming best
response of \ct; equivalently, a perfect equilibrium in the
corresponding strategic-form game.

Notice that designing the optimum $f$, given $h$, $\cD$ and $c$, for a finite $X$, is a
conventional combinatorial optimization problem.   We seek to label the points in $X$
with $\pm1$ so that the expectation, over $\cD$, of $h(x)\cdot f(\Delta(x))$ is maximized.
Here, $\Delta(x)$ is a best move of \ct, that is,
\begin{equation}\label{eq:bestresponse}
 \Delta(x) = \argmax_{y \in X} f(y) - c(x,y). 
\end{equation}
We note that $\Delta(x)$ may not be well-defined, if there are multiple $y$ which attain the maximum.  In the following, we assume that \ct may move to any of them; for simplicity, we do assume that if one of the maximum-attaining $y$ is $x$ itself, then $\Delta(x) = x$.  That is, if \ct is indifferent between moving and not moving, he will default to not moving.
We refer to the best payoff for \jr in the above full-information game at the ``strategic maximum" of the game:
\begin{definition}[Strategic Maximum]
The \emph{strategic maximum} in the full-information game is defined as
\[ 
\mathrm{OPT}_h(\cD, c) 
= \max_{f\colon X \to\signs} \Pr_{x \sim \cD} \inbrak{ h(x) = f(\Delta(x)) },\]
where $\Delta(x)$ is defined as in~\eqref{eq:bestresponse}.  Notice that $\Delta(x)$ depends on $f$.
\end{definition}
\begin{remark} For intuition, notice that if $c(x,x) = 0$ (that is, it costs nothing for \ct to stay where he is), then $\Delta(x)$ has the following characterization:
\begin{itemize}
\item 
if $f(x) = 1$, then $\Delta(x)=x$;
\item 
if $f(x)=-1$, let $y=\argmin_{y\in X\colon f(y)=1} c(x,y)$; then
\[ \Delta(x) = \begin{cases} y & c(x,y) < 2 \\ x & c(x,y) \geq 2.\end{cases} \]
\end{itemize}
Indeed, since \ct is best-responding, he only makes a move from input $x$ to
point $y$ if $c(x,y)$ is strictly
less than $2$, which is the payoff he obtains by improving his outcome 
from ``rejected" to ``accepted."
In this case, the quantity $f(\Delta(x))$ in the definition of the strategic maximum becomes
\[ f(\Delta(x)) = \max_{y: c(x,y) < 2} f(y). \]
\end{remark}
In Section~\ref{sec:npcomplete} we show that, for general cost functions, the
strategic maximum is NP-hard to approximate.
However, we will also show that for a natural class of cost functions, 
it is possible to to design a classifier for which \jr's payoff is arbitrarily close to the strategic
maximum, even when \jr has \em incomplete \em information.
To formalize this, we introduce
a second game, which we call the \em statistical classification game. \em
  In this game,  \jr does not know the 
target classifier $h$ for every point in $X$, but instead is given a few labeled
examples from an unknown distribution $\cD$. \ct best-responds to \jr's published classifier $f$.
\begin{definition}[Statistical Classification Game]\label{game:seqsamples}
 The players are \jr and \ct.  Fix a \em population~\em~$X$ and a
probability distribution $\cD$ over $X$.  
Fix a \em cost function \em $c:X \times X \to \R_+$ and a target classifier $h:X \to\signs$.
 \begin{enumerate}
 \item 
 \jr (who knows only the cost function $c$) can request labeled examples of the form $(x,h(x))$, 
 with $x$ being drawn from $\cD$.  
 She publishes a classifier $f: X \to\signs$.  
 \item 
 \ct (who knows $c$ and $f$), produces a function $\Delta : X \to X$.
 \end{enumerate}
 The payoff to \jr is 
 $\Pr_{x \sim \cD} \inbrac{ h(x) = f(\Delta(x)) }$.
 The payoff to \ct is 
 $\E_{x \sim \cD} [f(\Delta(x)) - c(x, \Delta(x))]$.
 \end{definition}


%

\subsection{Strategy-robust learning}
 
A learning algorithm in our setting has to accomplish two goals. 
First, it needs to learn the unknown target classifier from labeled examples.
Second, it needs to achieve high payoff for \jr in the
statistical classification game, by anticipating \ct's best response. 
%
%
Below, we give two definitions of stategy-robust learning which combine these goals; the second is a stronger requirement than the first.  
In our first definition, we fix an unknown target classifier $h$, and demand an algorithm which, with high probability over the samples, returns a classifier $f$ guaranteeing a near-optimal payoff to \jr in the statistical classification game.
In our second definition, we present a uniform notion: the learning algorithm must, with high probability, return a classifier that is guaranteed to work on \em any \em target classifier $h$ in some concept class $\cH$.
\begin{definition}[Strategy-robust learning]\label{def:srlearn}
Let $\cC$ be a class of cost functions. 
We say that an algorithm $\cA$ is a \emph{strategy-robust} learning algorithm 
for $\cC$ if the condition that follows holds.
For all distributions $\cD$, for all classifiers $h$, 
all $c\in\cC$ and for all $\eps$ and $\delta$, given a description of $c$ and access to labeled
examples of the form $(x, h(x))$, where $x \sim \cD$, $\cA$ produces a
classifier $f:X \to\signs$ so that, with probability at
least $1 - \delta$ over the samples, 
\begin{equation}\label{eq:nearoptnonunif} 
\Pr_{x \sim \cD} \inbrak{ h(x) = f(\Delta(x)) } 
\geq  OPT_h(\cD, c) - \eps.  
\end{equation} 
where $\Delta(x)$ is defined as in \eqref{eq:bestresponse}.
\end{definition}
One might expect, in line with PAC-learning~\cite{Valiant84}, that Definition \ref{def:srlearn}
might restrict $h$ to be in some concept class $\cH$.  
However, we will show that for a natural class $\cC$ of cost functions,
in fact it is possible to achieve strategy-robust learning with no dependence on $h$!  

However, we may want to ask a bit more.  Suppose that \jr builds a classifier for some property, and later wants to re-use the data to build a classifier for a slightly different property.  For example, returning to the scenario from the introduction, suppose that the school admissions board collects data on students and tries to predict academic success.  Later, the board is charged with recruiting to maximize the quality of the basketball team; they would like to use the same dataset to predict who will be a good student-athlete.  Later still, suppose that the this data set is made public, and many other schools try to use it to predict many things.  If enough different classifiers are trained on this data, the guarantee of Definition~\ref{def:srlearn} starts to degrade.  A strategy-robust learning algorithm should succeed with high probability on a single classifier, but there are no guarantees (beyond what the union bound gives) if it is used repeatedly.  This situation motivates the following definition.
\begin{definition}[Uniform strategy-robust learning]
Let $\cH$ be a concept class and $\cC$ be a class of cost functions. 
We say that an algorithm $\cA$ is a \emph{uniform strategy-robust} learning algorithm 
for $(\cH,\cC)$ if the condition that follows holds.  
For all distributions $\cD$, for all $c \in \cC$ and  
for all $\eps$ and $\delta$,
with probability at least $1 - \delta$ over draws $x \sim \cD$, the following holds simultaneously for all $h \in \cH$.
Given a description of $c$ and access to labels $(x,h(x))$,
$\cA$ produces a
classifier $f:X \to\signs$ so that
\begin{equation}\label{eq:nearopt} 
\Pr_{x \sim \cD} \inbrak{ h(x) = f(\Delta(x)) } 
\geq  OPT_h(\cD, c) - \eps,
\end{equation} 
where $\Delta(x)$ is defined as in \eqref{eq:bestresponse}.
\end{definition}
We will typically specify the number of
labeled examples that the algorithm requires as a function of
$\epsilon$, $\delta$ and a parameter that depends on the domain size (e.g., the
number of features).
%
%
\subsection{Our contributions}

Our main result is a strategy-robust learning algorithm, which comes with both uniform and non-uniform guarantees.
Our algorithm is computationally efficient when the cost function comes from 
a broad class of functions that we call \emph{separable}.
In the non-uniform case, the target classifier $h$ can be anything.
In the uniform case, the algorithm is efficient as long as the concept class $\cH$ is statistically learnable, but
it notably does not require that $\cH$ be efficiently learnable. 

Separable cost
functions are functions of the form $c(x,y) = \max\{0, c_2(y) - c_1(x)\}$, where
$c_1$ and $c_2$ are arbitrary functions, mapping the domain $X$ into the real
numbers. We take the maximum with~$0$ to obtain a nonnegative cost function.
We will later see and discuss a number of natural examples of separable cost
functions.

Our main theorem, and our stronger result, is about uniform strategy-robust learning.
\begin{theorem}[Informal]\label{thm:informal}
Let $\cH$ be a concept class that is learnable from $m$ examples up to error $\epsilon$ and
confidence $1-\delta$,  and let $\cS$ be the
class of separable cost functions.  
Then, there is a uniform strategy-robust learning
algorithm for $(\cH,\cS)$ with running time and sample complexity
$\poly(m,1/\epsilon,\log(1/\delta))$.
\end{theorem}
In fact, (the formal statement of) this theorem implies a non-uniform result: 
\begin{theorem}[Informal]
Let $\cS$ be the class of separable cost functions.
There is a non-uniform strategy-robust learning algorithm for $\cS$
with polynomial running time and sample complexity.
\end{theorem}

Our main theorem (and the non-uniform corollary) can be extended to a more general class of cost functions, which are obtained by
taking the minimum of $k$ separable cost functions.  We state only the uniform version here, the non-uniform version follows similarly. 

\begin{theorem}[Informal]\label{thm:informalkmin}
Let $\cH$ be a concept class that is learnable from $m$ examples up to error $\epsilon$ and
confidence $1-\delta$, and let $\cS^{(k)}$ be the
class of minima of $k$ separable cost functions.  
Then, there is a uniform strategy-robust learning
algorithm for $(\cH,\cS^{(k)})$ with sample complexity
$\poly(m,k,1/\epsilon,\log(1/\delta))$ and running time 
$\poly(m,\exp(k),1/\epsilon,\log(1/\delta)).$
\end{theorem}
Theorem~\ref{thm:informalkmin} applies to a broad class of cost functions:
it is not hard to see that any cost 
function on a finite domain~$X$ can be written as a minimum of separable cost
functions.  
Of course, the sample complexity in Theorem~\ref{thm:informalkmin} depends on $k$, the number of cost functions involved.
For general cost functions, $k$ grows with $|X|$ and might be quite large.  However, many spaces
admit a more efficient representation---for instance, if the cost function defines a metric that
admits a small $\epsilon$-net, $k$ depends only on the size of the net. 
Thus, $k$ is a parameter that interpolates nicely
between tractable cases where $k$ is small and the general case where $k$ is
unrestricted. 

The fact that the sample complexity in Theorem~\ref{thm:informalkmin} might be large is unavoidable:
for general cost functions, we have the following negative result.  

\begin{theorem}[Informal]
There is a class of metrics $\cal S$ such that, unless P = NP, there is no efficient strategy-robust learning algorithm for ${\cal S}$ that achieves expected payoff within $\epsilon = 1/|X|^\eta$ of the optimum, for any constant $\eta > 0$.
\end{theorem}
Recall that a
distance function is a {\em metric} if it is non-negative, symmetric, and satisfies the triangle 
inequality.
This result is an immediate corollary of the fact (which we will prove in Section~\ref{sec:npcomplete}) that approximating the strategic 
maximum for metrics is NP-complete. 

\subsubsection{Experimental evaluation}

We experimentally evaluate our framework on real data from a Brazilian social
network called Apontador. The data set deals with instances of review spam and
was recently studied in the context of spam fighting~\cite{CostaMBB14}. 
Classification of spammers is a natural setting for our methods, because spammers 
will of course try to game any automated attempt to identify them.
We model a cost
function that roughly reflects the loss in revenue that a spammer experiences
when changing certain attributes. For instance, when a spam message contains
a URL pointing to malware, it is costly for the spammer to remove this URL
from his message as his message loses its intended purpose. Acknowledging that
the modeling of a cost function can never be perfectly realistic, we evaluate
our approach while explicitly taking into account several types of modeling
inaccuracies. Specifically, we only assume that our cost function is roughly
correct and that the amount of gaming is possibly below or above the
threshold predicted by our theoretical framework. Our empirical observations
demonstrate that even in the presence of significant modeling errors and only a
small amount of gaming, our algorithm already outperforms a standard SVM
classifier. Complementing our robustness analysis, we explore an approach for
creating hybrid classifiers that interpolate between our classifier and
standard classifiers that aren't by themselves strategy-robust. We observe that
such hybrids often achieve an excellent trade-off between resilience to gaming
and classification accuracy.

\subsection{Related work}

The deterioration of prediction accuracy due to unforeseen events is often
described as \emph{concept drift} and arises in a number of contexts.
A sequence of works on \emph{adversarial learning} is motivated by the
question of learning in the presence of an adversary that tampers with the
examples of a learning algorithm. 
Typical application examples in this line
of work include intrusion detection and spam fighting. 
Early works considered
zero-sum games~\cite{Dalvi04} which are not very applicable to our problem as
there are almost always cases where the payoff should be high for both players
(e.g, a good student being admitted to a good college). 
More recent work
considers alternative game-theoretic notions~\cite{BS09,BS11,BKS12,GSBS13}. 
The most closely related is the work by Br\"uckner and Scheffer~\cite{BS11},
which considered a Stackelberg competition for adversarial learning. 
A notable difference with our setup is that they define the equilibrium with respect to
the sample, while we define it with respect to the underlying distribution. 
Our definition requires us to provide generalization bounds.
Beyond this difference, Br\"uckner and Scheffer focus on learning centered
linear classifiers when the Euclidean squared norm is the cost function.
The Euclidean norm is not separable and so our results are incomparable.
Stackelberg competitions have also been studied extensively in the context of
security games~\cite{KYKCT11,KCP10}.

%% file: separable.tex
\newcommand{\err}{{\mathop{\rm err}}}
\renewcommand{\thresh}[2]{#1[#2]}
\newcommand{\oneoverm}{{\mbox{$\frac{1}{m}$}}}

We begin by studying the class of \em separable \em cost functions, which arise
naturally in the context of gaming.  
To motivate the definition, recall the example of   
the school board which wants to exploit the correlation between parents' books
and students' performance.  In this (admittedly rather stylized) example, the
cost to \ct from moving from a household $x \in X$ with $50$ books to a
household $y \in X$ with $100$ books is simply the cost of the the additional
books. 

More generally, this logic applies to any situation where \ct can assign a
cost to each state $x \in X$, independently of how it was reached.  If the
cost of a state $x$ is $g(x)$, then the cost to \ct of moving from $x$ to $y$
is simply any additional cost: $c(x,y) = \max\inset{ 0, g(y) - g(x) }$.
For example, suppose that \jr is designing a spam filter, and \ct wishes to
send an email. Independently of the spam
filter, \ct wants his message to serve a purpose such as advertising or
distributing malware. We can assign a score $g(x)$ to each message in $x\in X$
that expresses how much utility the spammer experiences when this message is
delivered without being classified as spam. For example, a message is significantly 
less useful for the spammer after the URL pointing to malware has been
removed. The expression $\max\{0,g(y)-g(x)\}$ then captures the loss in
utility (or expected revenue) when moving from $x$ to $y.$ We will return to
this example in detail in our experimental evaluation in
Section~\ref{sec:experiments}.


With these examples in mind, we define a \em separable \em cost function as follows.
\begin{definition}\label{def:sep}
A cost function $c(x,y)$ is called separable if it can be written as 
\[ c(x,y) = \max \inset{ 0, c_2(y) - c_1(x)}, \]
for functions $c_1,c_2: X \to \R$ satsifying $c_1(X) \subset c_2(X)$.
\end{definition}
Above, the term ``separable" is a slight abuse of terminology, because the cost
function cannot be negative, and because of the assumption about $c_1(X) \subset c_2(X)$; a truly ``separable" function would be of the form
$c_2(y) - c_1(x)$, for arbitrary $c_1,c_2$.  However, we will stick with it for simplicity of
exposition.  
The two extra conditions are natural for cost functions.  The maximum with $0$ ensures that the cost function is non-negative.  The condition $c_1(X) \subset c_2(X)$ means that there is always a $0$-cost option (that is, Contestant can opt not to game, and can pay nothing).

Another important special case of a separable cost functions are linear cost functions of the form
\[ c(x,y) = \ip{ \alpha }{(y - x)}_{+}, \]
for $\alpha \in \R^n$.  With this cost function, each attribute can be increased independently at some linear cost, and can be decreased for free.  
For our arguments that follow, a linear cost function is helpful for intuition.


Our main result is that for separable cost functions, 
there is a nearly optimal algorithm for \jr, with a uniform guarantee.
The sample complexity and running time of this algorithm depend on the Rademacher complexity of the class $\cH$ of classifiers.
\begin{definition}
For a class $\mathcal{F}$ of functions $f:X \to \R$, the Rademacher complexity of $\mathcal{F}$ with sample size $m$ is defined as
  \[ R_m(\mathcal{F}): = \EE_{x_1,\ldots,x_m \sim \cD} \EE_{\sigma_1,\ldots, \sigma_m} 
  \bigg[ \sup\bigg\{ \oneoverm \sum_{i=1}^m \sigma_i f(x_i)  \,:\, f \in \mathcal{F} \bigg\} \bigg]~, \]
 where $\sigma_1,\ldots,\sigma_m$ are i.i.d. Rademacher random variables.
 \end{definition}
 Our algorithm, given below as Algorithm \ref{alg:linear}, has the following uniform
guarantee.
 \begin{theorem}\label{thm:separable}
Suppose the cost function $c$ is separable, i.e., 
$c(x,y) = \max\{0,c_2(y) - c_1(x)\}$ and $c_1(X) \subseteq c_2(X)\mper$
Let $\cH$ be a concept class, and let $\cD$ be a distribution.
Let $m$ denote the number of samples in Algorithm \ref{alg:linear}, and suppose
\[R_m(\cH) 
  + 2 \sqrt{\mbox{$ \frac{\ln(m+1)}{m} $}}
  +   \sqrt{\mbox{$ \frac{ \ln(2/\delta)}{8m} $}} \leq \frac{\eps}{8}~ . 
 \]
Under these conditions, with probability at least $1 - \delta$,
\eqref{eq:nearopt} holds for all $h \in \cH$.
\end{theorem}

Notice that Theorem~\ref{thm:separable} indeed implies the ``informal" version, Theorem~\ref{thm:informal}.
That is, if $\cH$ is statistically learnable (i.e., $R_m(\cH)$ decays inversely polynomially with $m$ for all distributions $\cD$, 
or sufficiently that the VC dimension of $\cH$ is bounded\footnote{Indeed, if $d$ is the VC dimension of $\cH$, we have
\[ R_m(\cH) \leq \sqrt{ \frac{2d \log(em/d) }{m}}\]
for all distributions $\cD$ (notice that $R_m(\cH)$ depends on $\cD$).}), 
then Algorithm~\ref{alg:linear} is a efficient, uniform strategy-robust learning algorithm for $\cH$.

It is worth pointing out that Algorithm \ref{alg:linear} is
\emph{computationally} efficient as long as $\cH$ has low sample complexity---even if $\cH$ itself is not 
computationally efficiently learnable!
As we mentioned above, the proof of Theorem~\ref{thm:separable} also implies that our algorithm satisfies the following non-uniform guarantee. 
\begin{corollary}\label{thm:nonunifseparable}
Suppose the cost function $c$ is separable.
Let $m$ denote the number of samples in Algorithm \ref{alg:linear}, and suppose that
\[  2 \sqrt{\mbox{$ \frac{\ln(m+1)}{m} $}}
  +   \sqrt{\mbox{$ \frac{ \ln(2/\delta)}{8m} $}} \leq \frac{\eps}{8}~ . 
 \]
Then with probability at least $1 - \delta$,
\eqref{eq:nearopt} holds for all distributions~$\cD$.
In particular, Algorithm~\ref{alg:linear} is an efficient (non-uniform) strategy-robust learning algorithm.
\end{corollary}
Corollary~\ref{thm:nonunifseparable} follows from Theorem~\ref{thm:separable} by setting $\cH = \inset{h}$, the singleton containing the fixed target classifier $h$.  Indeed, in this case $R_m(\cH) = 0$.

Before proving Theorem~\ref{thm:separable}, we state the algorithm and discuss the intuition behind it. 
In Figure \ref{fig:linear}, we illustrate the idea for a linear cost function, $c(x,y) = \ip{\alpha}{y-x}_+$.  
Because moving perpendicularly to $\alpha$ is free for \ct, 
\jr may as well choose a classifier $f$ that accepts some affine halfspace whose normal is equal to $\alpha$ 
(see Figure \ref{fig:linear}).  
Thus, the only issue is finding the correct shift for this halfspace.  
Because the calculated shift can only be based on samples, we choose the shift that is empirically the best. 
The latter can be calculated quickly because it is a one-dimensional problem. 


For a more general separable cost function 
\[ c(x,y) = \max\{0, c_2(y) - c_1(x)\}~,\]
 by the same argument,
 \jr may as well return a classifier $\thresh{c_2}{t}$ of the form:
\[  \thresh{c_2}{t}(x) = \begin{cases} 
             ~~1 & \text{if}~~ c_2(x) \geq t \\ 
              -1 & \text{if}~~ c_2(x) < t 
              \end{cases} ~~~~~~~~~~~~~~~~(x\in X)\]
for some $t$.
Algorithm \ref{alg:linear} gives the details, and we proceed with the proof below.

\begin{figure}
\begin{center}
\begin{tikzpicture}[scale=.9]
\begin{scope}
\clip (-4,-1) rectangle (4,4);
\draw[fill=gray!20,draw=white] (-3,5) -- (5, -3) -- (4,4) -- (-3,4);
\draw (-2,4)--(4,-2);
\draw[fill=gray] (3,3) circle (2.8cm);
\draw[dashed] (-4,4) -- (2,-2);
\draw[xshift=-2cm,yshift=0cm,->] (0,0) to node[left,pos=.8, node distance=.5cm] {$\alpha$} (.5,.5);
\draw[dashed] (-6, 4) -- (2, -4);
\draw[dashed] (-4, 0) -- (0,-4);
\node at (2.5,2.5) {$f$};
\node at (-.5,3.5) {$f'$};
\end{scope}
\node[draw,circle, fill=black, scale=.3, label={$y$}] at (1,1) {};
\node[draw,circle, fill=black, scale=.3, label={$x$}] at (0,0) {};
\node[draw,circle, fill=black, scale=.3, label={$x'$}] at (-1,1) {};
\end{tikzpicture}
\caption{Suppose the optimal classifier for \jr is $f$ (which accepts the dark gray region), and the cost function is $c(x,y) = \ip{\alpha}{y-x}_+$.  
Because moving perpendicular to $\alpha$ is free for \ct, then the payoff for \jr if she plays $f'$ (shown above, which accepts the light gray region) is the same as her payoff if she plays $f$.
Indeed, suppose that the agent $x$ shown above would be willing to move to $y$ to get accepted by $f$.  Then $x'$ would also be willing to move to $y$, because the cost is the same.
Thus, \jr may restrict his or her search to classifiers $f'$ that accept all points in some affine halfspace whose normal is equal to $\alpha$.}
\label{fig:linear}
\end{center}
\end{figure}
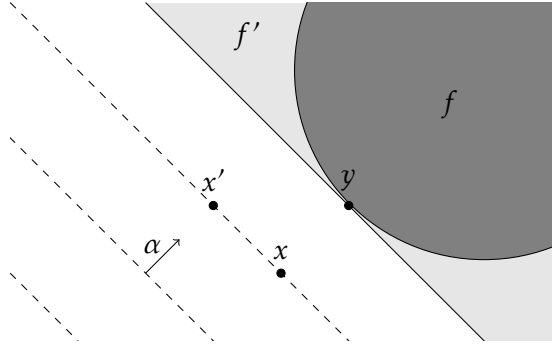

\begin{algorithm}
\textbf{Inputs:}{ Labeled examples $(x_1,h(x_1)), \ldots, (x_m, h(x_m))$ from $x_i \sim \cD$ i.i.d..  Also, a description of a separable cost function $c(x,y) = \max\{0,c_2(y) - c_1(x)\}$.  }

For $i =1,\ldots,m$, let
 \begin{align*} t_i &:= c_1(x_i) \\
	s_i &:= \begin{cases} \max\inparen{ c_2(X) \cap [t_i,t_i+2] }  & c_2(X) \cap [t_i,t_i+2] \neq \emptyset \\ \infty & c_2(X) \cap [t_i,t_i+2] = \emptyset ~. \end{cases} 
\end{align*}
For convenience, set $s_{m+1} = \infty$.
 
Compute
\[\widehat{\mbox{\rm err}}(s_i) := \oneoverm \sum_{j=1}^m  \ind{ h(x_j)  \neq \thresh{c_1}{s_i - 2}(x_j) }.\]

Find $i^*$, $1\le i^* \le m+1$, that minimizes $\widehat{\mbox{\rm err}}(s_i)$.

\textbf{Return:}{ 
$f := \thresh{c_2} {s_{i^*}}$. 
}
\caption{$\cA$: gaming-robust classification algorithm for separable cost functions}
\label{alg:linear}
\end{algorithm}
\begin{remark}[Input to Algorithm~\ref{alg:linear}]\rm
Algorithm \ref{alg:linear} takes a cost function $c(x,y)$ 
$= \max\{0,c_2(y) - c_1(x)\}$ as an input, and it returns some threshold function based on $c_2$.  We have been a little sloppy about how exactly $c$ should be represented.  A quick inspection of the algorithm shows that in order to compute the threshold, $\A$ needs only black-box access to $c_1$, and enough access to $c_2$ to determine $c_2(X) \cap [t_i, t_i+2]$.  In order to return the classifier $f$, $\A$ additionally needs whatever access to $c_2$ it is expected to return.  For example, if we only ask that $\A$ be able to provide black-box access to $f$, then black-box access to $c_2$ suffices for this step.  If we ask that $\A$ return a short description of $f$, then a short description of $c_2$ suffices for this step.  
\end{remark}

\input{sepproof}

%% file: sepproof.tex
\begin{proof}[Proof of Theorem \ref{thm:separable}]
%
Assume for simplicity that the cost function satisfies $c(x,y) \neq 2$, for all $x,y\in X$.  
First, for any mapping  $f:X \to \{-1,1\} $, define
\begin{align*}
 \Gamma(f) 
  &:=  \big\{\, x \suchthat \max\{f(y)\,:\, c(x,y) < 2 \} = 1 \,\big\} \\ 
  &=   \big\{\, x \suchthat (\exists y \in X)(f(y) = 1,~ c(x,y) < 2)\, \big\} \\ 
  &=   \big\{\, x \suchthat c_1(x) > \min\{ c_2(y)\,:\, f(y) = 1\}  - 2\, \big\}~. 
 \end{align*}
\begin{claim}\label{claim:notasobviousasyouthink}
 $\Gamma(f)$ is the set of $x \in X$ such that $f(\Delta(x)) = 1$ when $\Delta$ is a best response of \ct.   
\end{claim}
\begin{proof}
 Indeed, for $x \in \Gamma(f)$, there exists some $y$ such that $f(y) = 1$, so that the payoff to \ct when he plays $\Delta(x) = y$ is equal to $1 - c(x,y) > -1$.  
 On the other hand, suppose that \ct plays $\Delta(z) \in X$ for some with $f(z) \not= 1$.  Then the best payoff of \ct is equal to $-1 - c(x,z) \leq -1$, because $c(x,z) \geq 0$.  
 So, the best response of \ct is to choose $\Delta(x) = y$ for some $y$ with $f(y) = 1$.
This establishes that $\Gamma(f) \subseteq \inset{ x \in X \suchthat f(\Delta(x)) = 1 }$.

For the other direction, suppose that $f(\Delta(x)) = 1$.  Then there is some $y \in X$ so that
\[1 - c(x,y) > -1 - \min_{z \in X} c(x,z) = -1, \]
using from the definition of separability that $c_1(X) \subseteq c_2(X)$, and hence for all $x$, 
\[\min_{z \in X} c(x,z) = \min_{z \in X} \max \inset{ 0, c_2(z) - c_1(x) } = 0. \]
In particular, $c(x,y) < 2$, and so $x \in \Gamma(x)$.  This establishes that
\[ \inset{ x \in X \suchthat f(\Delta(x)) = 1 } \subseteq \Gamma(f), \]
and proves the claim.
\end{proof}
Claim~\ref{claim:notasobviousasyouthink} is the only place in the proof where we need either of the extra conditions in Definition~\ref{def:sep} (that $c(x,y) \geq 0$ and $c_1(X) \subseteq c_2(X)$).

 Given this characterization of $\Gamma(f)$, we next argue that we may replace $f$ by a much more structured function $f'$ so that $\Gamma(f) = \Gamma(f')$; in particular, the payoff to \jr under $f$ will be the same as under $f'$, and so we can restrict our attention to these more structured functions.
 For any $f$, let
\begin{equation}\label{eq:simplerform}
 f'(y) = \begin{cases}~~ 1 & \text{if}~~ c_2(y) \geq \min \{ c_2(z)\,:\, f(z) = 1 \}   \\ 
       -1                & \text{otherwise~.} 
          \end{cases}
\end{equation}
 Then we have
 
 \begin{align*}
   \Gamma(f) &=  \inset{\, x \suchthat c_1(x) > \min\{ c_2(y) \,:\, f(y) = 1 \,\} - 2 \,} \\
 &= \inset{ x \suchthat c_1(x) > \min\{ c_2(y)\,:\,  f'(y) = 1 \,\}  - 2 \,} \\
 &= \Gamma(f')~.
 \end{align*} 
 In particular, for any true classifier $h \in \cH$, the payoff to \jr if she plays $f$ is the same as if she plays $f'$:   
\begin{align*}
   {\mathbb{P}} \left\{ h(x)\right.   =  \left. \max\{\, f(y)\,:\, c(x,y)<2\, \} \, \right\} 
  =&\  \PR{x \in \inparen{ \Gamma(f) \triangle \inset{ y\,:\, h(y) = 1 } }^c \,} \\
  =&\  \PR{x \in \inparen{ \Gamma(f') \triangle \inset{ y\,:\, h(y) = 1 } }^c \,}  \\
  =&\  \PR{\, h(x) = \max \{\, f'(y) \,:\, c(x,y) < 2\,\}\, } ~.
 \end{align*}
Above, $\triangle$ denotes symmetric difference.
Thus, it suffices to consider classifiers of the form of \eqref{eq:simplerform}. 
That is, our classifier may as well be equal to
$\thresh{c_2}{s}$, for some $s \in c_2(X) \cup \inset{\infty}$, where $s$ plays the role of 
$\min\{\, c_2(z)\,:\, f(z) = 1 \,\} $, 
and $s = \infty$ means that there is no $z$ such that $f(z) = 1$.
Let
\[ S := c_2(X) \cup \inset{\infty} \]
 be the set of these relevant values of $s$.  Recall the definition of $s_i$ from Algorithm~\ref{alg:linear}.
 For $s \in S$, we have\footnote{ As usual, $\infty - 2 = \infty$. }
 \[ \Gamma(\thresh{c_2}{s}) = \inset{ \,x \suchthat c_1(x) > s - 2\, }~. \]
The best possible payoff to \jr is obtained by finding the best threshold $s$, i.e.,
 \[  OPT_h(\cD, c) = 1 - \inf \{\, \err(s) \,:\, s \in S\, \}~,  \]
where $\err(s) :=  \PR{ h(x) \neq \thresh{c_1}{s-2}(x) }.$
In Algorithm~\ref{alg:linear}, \jr returns $f = \thresh{c_2}{s_{i^*}}$, 
and as above
the payoff to \jr from this $f$ is equal to
\[ \PR{ h(x) \neq \thresh{c_1}{s_{i^*} - 2}(x) } = 1 - \err(s_{i^*})~. \]
Thus, to prove Theorem~\ref{thm:separable}, it suffices to show that for all $h \in \cH$,
\begin{equation}\label{eq:want}
 \err(s_{i^*}) ~\leq~ \inf \{\, \err(s)\,:\, s\in S\, \} + \eps~. 
\end{equation}
To establish this, we first observe that there is no loss of generality in Algorithm~\ref{alg:linear} 
by considering only the $s_i$, $i=1,\ldots,m+1$, where as in Algorithm~\ref{alg:linear} we set $s_{m+1} = \infty$. 
 \begin{claim}\label{claim:1}
 \begin{align*}
  \widehat{\mbox{\rm err}}(s_{i^*})  
  =&\ \min \{\,\widehat{\mbox{\rm err}}(s_i)\,:\, i=1,\ldots,m+1\,\}  \\
  =&\ \inf \{\,\widehat{\mbox{\rm err}}(s)\,:\, s \in S\, \}~.
  \end{align*}
 \end{claim}
 \begin{proof}
The first equality is just the definition of $i^*$.  
The second equality follows from the fact that
 \[ \widehat{\mbox{\rm err}}(s) = \oneoverm \sum_{j=1}^m \ind{ h(x_j) \neq \thresh{ c_1 }{ s - 2 }(x_j) } \] 
 only changes when $\thresh{ c_1}{s-2}(x_j)$ changes for some $j$. 
Thus, by construction, this sum takes on every possible value (as $s$ ranges over $S = c_2(X) \cap \inset{\infty}$) at the points $s_i$,
$i=1,\ldots,m+1$.
 \end{proof}  
  \begin{claim}\label{claim:2}
  With probability at least $1 - \delta$,  for all $h \in \cH$ and for all $s \in S$,
 \[ 
\inabs{ \widehat{\mbox{\rm err}}(s) - \err(s) }~ \leq~ 4R_m(\cH) 
      + 8 \sqrt{ \mbox{$\frac{\ln(m+1)}{m}$}} 
      +   \sqrt{ \mbox{$\frac{ 2 \ln(2/\delta)}{m}$} }~.
\]
In particular, under the conditions of Theorem~\ref{thm:separable}, with probability at least $1 - \delta$,
\[ \sup \left\{\, \inabs{ \widehat{\mbox{\rm err}}(s) - \err(s)\,} \,:\, h \in \cH,~ s \in S \,\right\} ~\leq~ \eps/2~. \]
 \end{claim}
 \begin{proof}
 Writing out the definition of $\widehat{\mbox{\rm err}}$ and $\err$, we need to bound the absolute value of the difference
\begin{align*}
    \widehat{\mbox{\rm err}   }(s) -  \err(s) 
  &=  \oneoverm \sum_{j=1}^m \ind{ h(x_j) \neq \thresh{c_1}{s-2}(x_j) } 
  - \EE_{x \sim \cD} \inbrak{ \ind{ h(x) \neq \thresh{c_1}{s-2}(x)} }
\end{align*}
simultaneously for 
all $h \in \cH$, $s \in S$.
 By standard arguments (see, for example, Theorem 3.2 in~\cite{BBL05}), for all $h \in \cC, s \in S$,
 \begin{equation}\label{eq:std}
 |\, \widehat{\mbox{\rm err}   }(s) -  \err(s)  \,|
  ~ \leq ~ 2 R_m(\mathcal{X}) 
   + \sqrt{\mbox{$ \frac{ 2\ln(2/\delta) }{m}$ } }~,
  \end{equation}
 where 
 \[ \mathcal{X} = \inset{\, h \cdot \thresh{c_1}{s-2} \suchthat h \in \cH, s \in S\, }~.\]
Thus, it suffices to control the Rademacher complexity of $\mathcal{X}$, which is in turn controlled by 
\begin{equation}\label{eq:additive}
R_m(\mathcal{X}) \leq 2 \inparen{ R_m( \cH) + R_m(\mathcal{Y})}, 
\end{equation}
 where $\mathcal{Y} = \inset{ \thresh{c_1}{s-2} \suchthat s \in S}$.
 Note that, because all the functions in $\mathcal{H} \cup \mathcal{Y}$ are $\pm 1$-valued,  
 \[ h(x) \cdot \thresh{c_1}{s-2}(x) = | \,h(x) + \thresh{c_1}{s-2}(x)\,| - 1  \]
 for every $x$.
 Inequality \eqref{eq:additive} follows from a contraction principle (see, e.g., Theorem 4.2 in~\cite{LT91}) and the definition of the Rademacher complexity.

It remains to bound $R_m(\mathcal{Y})$.  
  Fix $x_1,\ldots, x_m \in X$ and sign flips $\sigma_i \in \{-1,1\}$.  As in the proof of Claim~\ref{claim:1},
all of the values that $\sum_{i=1}^m \sigma_i \thresh{c_1}{s-2}(x_i)$ takes on as $s$ ranges over $S$ 
are attained at $\{s_1,\ldots,s_{m+1}\}$.  
Thus, for fixed $x_1,\ldots,x_m \in X$, using a Chernoff bound and the union bound, and integrating to bound the expectation,  we obtain
 \begin{align*}
 \EE_{\sigma}\big[& \sup \big\{  \oneoverm \sum_{i=1}^m \sigma_i \thresh{c_1}{s-2}(x_i)\,:\, s\in S \big\} \big] \\
 &= \EE_\sigma \big[ \sup\big\{  \oneoverm \sum_{i=1}^m \sigma_i \thresh{c_1}{s_j-2}(x_i)\,:\,  j=1,\ldots,m+1  \big\}\big] \\
 & \leq 2 \sqrt{\mbox{$ \frac{\ln(m+1)}{m}$} }~.
 \end{align*}  
 Thus, we have
 \[ R_m(\mathcal{Y}) \leq 2\sqrt{\mbox{$\frac{\ln(m+1)}{m}$}}, \]
 and altogether inequality (\ref{eq:std}) implies that for all $h \in \cH$ and $s \in S$,
  \[ 
  \inabs{ \, \widehat{\mbox{\rm err}   }(s) -  \err(s) \, }
   \leq 4\inparen{ R_m(\mathcal{\cH}) + 2 \sqrt{\mbox{$\frac{\ln(m+1)}{m}$}}\,} + \sqrt{ \mbox{$\frac{ 2\ln(2/\delta) }{m}$} }~,
   \] 
 which completes the proof of the claim.
 \end{proof}
 Claims \ref{claim:1} and \ref{claim:2} establish Theorem \ref{thm:separable}.  
 Indeed, we have, with probability at least $1 - \delta$, for all $h \in \cC$,
 \begin{align*}
 \err(s_{i^*}) 
  \leq &\ \widehat{\mbox{\rm err}}(s_{i^*}) + \eps/2  \\
     = &\ \inf \, \left\{ \widehat{\mbox{\rm err}}(s)\,:\, s\in S \right\} + \eps/2 \\
  \leq &\ \inf \, \{  \err(s)\,:\, s \in S \} + \eps~,
 \end{align*} 
establishing inequality \eqref{eq:want} and completing the proof.
\end{proof}

%% file: general.tex
While separable cost functions are quite reasonable, they do not capture everything.  
In this section, we consider more general cost functions.  
We extend Algorithm \ref{alg:linear} to work for a cost function that is the minimum of an arbitrary set of separable cost functions.  
This is a much broader class.
In fact, \em every \em cost function can be represented as the minimum of separable cost functions, although not necessarily very parsimoniously.
\begin{proposition} \label{prop:minqs}
Let $X$ be any finite set and let $c: X \times X \to \R$ be any mapping. Suppose
\[ D \geq  \max \{ c(x,y)\,:\, x,y \in X  \} ~. \]
Under these conditions,
\[ c(x,y) = \min \inset{ c(w,z) + D\cdot \ind{ x\neq w } + D\cdot \ind{y \neq z} \,:\, w,z \in X }~. \]
\end{proposition}
Since each of the cost functions $c_{w,z}(x,y) = c(w,z) + D\cdot \ind{x \neq w} + D\cdot \ind{y \neq z}$ is a separable cost function,
Proposition~\ref{prop:minqs} implies that any $c$ can be written as the minimum of $|X|^2$ cost functions. 
The sample complexity of our extension depends on the number of cost functions;
since $|X|$ may be quite large (possibly exponential in the parameter of interest), Proposition~\ref{prop:minqs} might not help.
However,  a smaller number of cost functions can be used if $X$ has nice geometric structure.
\begin{proposition}
Let $X$ be any finite set and let $c: X \times X \to \R$ be a metric. 
Let $S$ be an $\eps$-net of $X$: that is, for every $x \in X$, there is some $s \in S$ so that $c(x,s) \leq \eps$. 
Under these conditions, for every $x,y\in X$,
\[ c(x,y) 
 \leq \min \inset{ c(x,w) + c(w,z) + c(z,y) \,:\, w,z \in S }
 \leq c(x,y) + 4\eps~ . \]
 \end{proposition}
Thus, when $c$ is a metric, our problem is very close to a problem where the cost function is the minimum of separable cost functions, and the number of cost functions we need to consider depends essentially on the covering number of the metric space $(X,c)$.

Algorithm \ref{alg:minsep} is an adaptation of Algorithm \ref{alg:linear} for cost functions of the form
\[ c(x,y) = \min  \{ b(x,y) \,:\,  b \in \mathcal{B} \} ~, \]
where each function $b\in \mathcal{B}$ is separable, i.e.,
\[ b(x,y) = \max\{0,b_2(y) - b_1(x)\} ~.\]
\begin{algorithm}[ht!]
\textbf{Inputs:}{ Labeled examples $(x_1,h(x_1)), \ldots, (x_m, h(x_m))$ from $x_i \sim \cD$ i.i.d..  Also, a description of $k$ separable cost functions $b(x,y) = \max\{0,b_2(y) - b_1(x)\}$ for $b \in \mathcal{B}$. }
 
For $i=1,\ldots,m$ and $b \in \mathcal{B}$, set
\begin{align*} 
t_{i,b} &= b_1(x_i) \\ 
s_{i,b} &= \begin{cases} 
\max\{ b_2(X) \cap [t_{i,b}, t_{i,b} + 2] \} &  \text{if}~~~b_2(X) \cap [t_{i,b}, t_{i,b} + 2 ] \neq \emptyset \\
\infty &  \text{if}~~~ b_2(X) \cap [t_{i,b}, t_{i,b} + 2] = \emptyset \end{cases}
\end{align*}
and set $s_{m+1,b} = \infty$ for all $b \in \mathcal{B}$.

 For each $\vec{s} \in \bigoplus_{b \in \mathcal{B}} \inset{ s_{i,b} : i=1,\ldots,m+1}$, compute
 \[ \widehat{\err}(\vec{s}) 
 := \oneoverm \sum_{j=1}^m \ind{ h(x_j) \neq 
   \min \{ \thresh{b_1}{\vec{s}_{b} - 2}(x_j) \,:\, b \in \mathcal{B}\} }. \]

Find a $\vec{s}^*$ that minimizes $\widehat{\err}(\vec{s})$.

\textbf{Return:}{ 
$f(x) = \min \{ \thresh{b_2}{\vec{s}^*_{b}}(x) \,:\, b \in \mathcal{B} \}$.
}
\caption{$\cA$: gaming-robust classification algorithm for minima of separable cost functions}
\label{alg:minsep}
\end{algorithm}
\begin{theorem}\label{lem:minsep}
Suppose the cost function $c$ is the minimum of separable functions,
\[ c(x,y) = \min \{ b(x,y)\,:\, b \in \mathcal{B}  \} ~,\]
where each $b:X \times X \to \R$ is separable.
Let $\cD$ be a distribution on $X$ and suppose that Algorithm \ref{alg:minsep} uses $m$ samples, so that $m$ satisfies 
\[R_m(\cH) 
+ 2 \sqrt{\mbox{$ \frac{|\mathcal{B}|\ln(m+1)}{m}  $}} 
+ \sqrt{\mbox{$ \frac{ \ln(2/\delta)}{8m} $} } ~\leq~\frac{\eps}{8}
~ . \]
Under these conditions, with probability at least $1 - \delta$, \eqref{eq:nearopt} holds for all $h \in \cH$ and for the distribution $\cD$.
The running time of Algorithm \ref{alg:minsep} is $O(m^{|\mathcal{B}|})$.
\end{theorem}
The intuition for Algorithm \ref{alg:minsep} is similar to that for Algorithm \ref{alg:linear}, and is illustrated in Figure \ref{fig:minsep} for the minimum of two linear cost functions.
The proof of Theorem \ref{lem:minsep} is also similar to that of Theorem \ref{thm:separable}; for completeness, we give it in Appendix~\ref{sec:minproof}.
\begin{remark}[Improvements for structured classes $\mathcal{B}$]
When the size of $\mathcal{B}$ is small, Theorem \ref{lem:minsep} gives a nice bound.  However, if $\mathcal{B}$ is large (as in our extreme example of the beginning of this section), these guarantees are not so good.  An inspection of the proof (in Appendix~\ref{sec:minproof}) shows that the term $\sqrt{ \frac{|\mathcal{B}|\ln(m+1)}{m} }$ may be replaced by $R_m(\mathcal{H})$, where
\[ \mathcal{H} = \inset{ \min_{b \in \mathcal{B}} \thresh{b_1}{\vec{s}_b - 2} \suchthat 
\vec{s} \in \bigoplus_{b \in \mathcal{B}} \inparen{ b_2(X) \cup \inset{\infty}} }. \]
For some sets $\mathcal{B}$ of separable cost functions, this may be much smaller.
\end{remark}
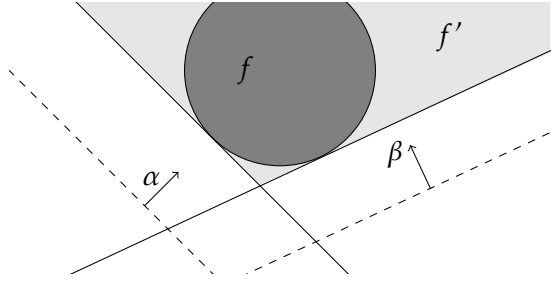
\begin{figure}
\begin{center}
\begin{tikzpicture}[scale=.9]
\begin{scope}
\clip (-4,-1) rectangle (4,3);
\draw[fill=gray!20,draw=white] (-5,5) -- (-.3, .3) -- (4,2.26) -- (4,4) --(-5,5);
\draw[fill=gray] (0,2) circle (1.41cm);
\draw(-4,4) -- (2,-2);
\draw[rotate=70,yshift=.55cm] (-4,4) -- (4,-4);
\draw[xshift=-2cm,yshift=0cm,->] (0,0) to node[left,pos=.8, node distance=.5cm] {$\alpha$} (.5,.5);
\draw[rotate=70,yshift=-2cm,xshift=1cm,->] (0,0) to node[left,pos=.8, node distance=.5cm] {$\beta$} (.5,.5);
\draw[dashed] (-6, 4) -- (2, -4);
\draw[rotate=70, xshift=-1cm,dashed] (-4, 4) -- (4,-4);
\node at (2.5,2.5) {$f'$};
\node at (-.5,2) {$f$};
\end{scope}
\end{tikzpicture}
\caption{Suppose the the optimal classifier for \jr is $f$ (which accepts the dark grey region), and the cost function is $c(x,y) = \min \inset{ \ip{\beta}{y-x}_+, \ip{\alpha}{y-x}_+}$.  
For the same reasoning as in Figure \ref{fig:linear}, the classifier $f'$ has the same payoff to \jr as $f$ does.  Thus, \jr may restrict his/her search to classifiers $f$ that are the intersections of two affine halfspaces.}
\label{fig:minsep}
\end{center}
\end{figure}

%% file: np_complete.tex
What happens when the cost function $c$ is not separable?  
It turns out that for general cost functions, 
any algorithm for \jr requires more
than polynomial time to obtain a near-optimum classifier, unless $P=NP$.
This holds true 
\begin{itemize}
\item[(a)]  even if the underlying distance function is a metric (another very natural class of cost function), and 
\item[(b)]  even if the learning algorithm were given correct labels $h(x)$ for \em all \em members $x \in X$
of the population,
\end{itemize} 
when the desired deviation $\eps$ is inverse-polynomially small and the distribution $\cD$ is uniform.
The above statements are consequences of the following result:
\begin{theorem}\label{thm:npcomplete}
Given a finite population $X$ with the uniform distribution, a metric
$c$ on $X$, and a 
target labeling $h: X\mapsto \{-1,+1\}$,  it is NP-hard to compute
the strategic optimum within $\epsilon = {1\over |X|^\eta}$ for any constant $\eta > 0$.
\end{theorem}

\input{npproof}

The metric constructed in the proof has ``separability dimension''
(the smallest number of separable functions needed to achieve it as a minimum)
that grows linearly with the population.  The same dimension appears in
the exponent of the running time of the algorithm of the previous section.
It is an interesting open problem to determine whether this exponential
dependence is inherent; the other possibility is that the problem is
{\em fixed-parameter tractable}  with respect to the ``separability dimension'' 
parameter.  We suspect that exponential dependence is necessary.

%% file: npproof.tex
\begin{proof}[Proof of Theorem~\ref{thm:npcomplete}]
We will reduce from {\sc 3Sat}.  Suppose we are given
a  {\sc 3Sat} Boolean formula with $n$ variables $x_1,\ldots,x_n$ and $m$ clauses $C_1,\ldots,C_m$, where $C_i$ has three literal occurrences $L_{i1},L_{i2},L_{i3}$.  
We now construct our instance of \textsc{Strategic Optimum} as follows.  We need to specify $X$, $h$, and $c$.
We begin by constructing a weighted population $Y$, which will consist of points $y$ and positive integer weights $w(y)$ for each $y \in Y$.
Our population $X$ will simply consist of $w(y)$ identical copies of each $y \in Y$.  Thus, $|X| = \sum_{y \in Y} w(y)$.
We will also specify labels $h(y)$ for each $y \in Y$, which the points $x \in X$ will inherit.  Fix a number $K$ (polynomial in $m$) to be chosen later.  Our weighted population $Y$ consists of:
\begin{itemize}
\item 
$3m$ points $L_{ik}$ for $1 \leq i \leq m$ and $k \in \inset{1,2,3}$, corresponding to the literal occurrences in the clauses.
These points each have weight $w(L_{ik}) = K(m-1-{1\over m})$ and label $h(L_{ik}) = -1$.
\item 
${m}\choose{2}$ points $P_{ij}$ for $1\leq i<j\leq m$ corresponding to unordered pairs $\{C_i,C_j\}$ of clauses.  
These points each have weight $w(P_{ij}) = 2K$ and label $h(P_{ij}) = +1$.
\item 
$9\cdot{{m}\choose{2}}$ points $Q_{ikj\ell}$, for $1\leq i<j\leq m$, and for $k,\ell \in \inset{1,2,3}$ so that $L_{ik}$ is {\em not} the negation of $L_{j\ell}$.  These points correspond to unordered pairs of literal occurrences $\{L_{ik},L_{j\ell}\}$ of literal occurrences in different clauses {\em which are not contradictory}.  
They have weight $w(Q_{ijk\ell}) = 1$ and label $h(Q_{ijk\ell}) = -1$.  (Actually, their label does not matter). 
\item 
One other point $R$ with a huge weight $w(R) = KM$, for a very large value $M$, and label $h(R) = -1$. 
Choose $M = 2{m \choose 2}$.
\end{itemize}

We next define a metric $c:X \times X \to \R_+$.  It will take only two nonzero values, $1.5$ and $2.5$.  Notice that this guarantees $c$ satisfies the triangle inequality.  We will choose $c$ so that $c(x,x) = 0$ and $c(x,y) = c(y,x)$, and so $c$ will indeed be a metric.
To describe $c$, it suffices to describe the points which are ``close," that is, which have distance $1.5$.
Further, it suffices to define $c$ for points in $Y$, and we will extend it to $X$ in a natural way: for points $x,x' \in X$, if they come from the same $y \in Y$, they will have distance $1.5$; if $x,x'$ come from $y\neq y'$ respectively, then $c(x,x') = c(y,y')$.  The close pairs of points in $Y$ are:
\begin{itemize}
\item All pairs of the form $\{P_{ij},Q_{ijkl}\}$;
\item All pairs of the form  $\{P_{ij},R\}$;
\item All pairs of the form  $\{Q_{ijk\ell},L_{ik}\}$ or $\{Q_{ijk\ell},L_{j\ell}\}$.
\end{itemize}
\begin{claim}\label{claim:npclaim}
If the given formula is unsatisfiable, the number of points labeled $+1$ by the \jr's optimum $f$ is equal to
\[ b = K\inparen{ M + 3m\inparen{ m - 1 - {1 \over m}} } + 9{m \choose 2} , \]
which we call the \em baseline \em payoff.
Otherwise, if the given formula is satisfiable, then there is a labeling $f$ of the points with payoff at least $b + K - 9{m \choose 2}$.
\end{claim}
\begin{proof}
In the following, we will consider a graph with vertices $Y$.   Two vertices $x,y$ are neighbors in this graph if $c(x,y)  = 1.5$.  Let $\Gamma(x)$ denote the neighbors of $x$ in this graph. Thus, the best-response $\Delta$ to a classifier $f$ is
\[ \Delta(x) = \begin{cases} x & f(x) = 1 \\ x & f(x) = -1 \text{ and } f(y) = -1 \forall y \in \Gamma(x) \\ y & f(x) = -1 \text{ and } f(y) = 1, y \in \Gamma(x) \end{cases}, \]
where above if $y$ in the last case is not uniquely defined \ct can pick any such $y$.
 
First observe that the baseline payoff is obtained by the classifier $f(x) = -1$ for all $x \in X$, and so it is certainly acheivable. 
We now argue that \jr can do better if and only if the original formula was satisfiable.  We make a few observations about \jr's optimal classifier $f$.
\begin{itemize}
	\item First, because of our choice of $M$, we must have $f(P_{ij}) = -1$ for all $i,j$.  
 Indeed, our choice implies that $KM > |X| - KM$; thus, if $f(P_{ij}) = 1$ for some $i,j$, then
 \ct will set $\Delta(R) = P_{ij}$, and \jr will mis-classify the point $R$, and get a payoff worse than the baseline.
	\item Next, $f(L_{ik}) = -1$ for all $i,k$.  Indeed, since $h(L_{ik}) = -1$ and $h(x) = -1$ for all of the ($Q$-type) neighbors of $L_{ik}$, there can be no benefit to \jr for making $f(L_{ik}) = +1$.  
	\item For each $P_{ij}$, at most one $Q$-point $Q_{ikj\ell}$ in $\Gamma(P_{ij})$ has $f(Q_{ikj\ell}) = +1$.  Indeed, each $Q$-point is connected to exactly one $P_{ij}$, and once one of them is accepted by \jr, she can gain nothing by accepting additional points of $\Gamma(P_{ij})$.
\end{itemize}
Thus, the optimal $f$ only assigns positive weights to $Q$ points, and it does so to at most one $Q$-point in each $\Gamma(P_{ij})$.  Suppose that $f(x) = +1$ for the set $A$ of $Q$-points, and let $B = \Gamma_L(A)$ be the set of $L$-points adjacent to $A$.  Now, the size of $B$ can vary based on how the literals overlap with the clauses.  It satisfies
\[ \frac{2|A|}{m-1} \leq |B| \leq 2|A|, \]
where the lower end is attained when there are complete collisions, and the upper end is attained when there are no collisions.
Now consider the number of points of $X$ that \jr classifies correctly under such an $f$.  It is
\[ K \inparen{ M + (3m - |B|)\inparen{m - 1 - {1\over m}} + 2|A| } + \inparen{ 9{m \choose 2} - |A| } = b + \delta, \]
where
\[ \delta = K \inparen{ |B|\inparen{m - 1 - {1 \over m }} + 2|A| } - |A|. \]
Consider this first term, which is multiplied by $K$.  This is only positive when $|B| = \frac{ 2|A| }{m-1}$ is as small as it can possibly be, which happens only if $|A| = {m \choose 2}$ and $|B| = m$.   In this case, the first term is equal to $K$, and we have $\delta = K - |A| \geq K - 9{m \choose 2}$.
But this happens if and only if we can choose $m$ different literals $L_{ik}$, one from each clause, so that no pair of them contradict each other; that is, if and only if the original formula was satisfiable.
\end{proof}
Now the theorem follows quickly from the claim.  We choose $K$ to be a large polynomial in $m$, say $m^{2/\eta}$ for some small constant $\eta$.
Thus, $|X|$ is on the order of $m^{2/\eta + 2}$.
Suppose there is a polynomial-time algorithm which approximates the strategic optimum up to $\eps$.
Claim~\ref{claim:npclaim} implies a contradiction for any 
\[ \eps < \frac{K - 9{m \choose 2} }{|X|} = \frac{ K - 9{m \choose 2} }{ K\inparen{ 3m + {m \choose 2} + M } + 9{m \choose 2} }. \]
Using our choice of $K$, for sufficiently large $m$ the right hand side is at least $|X|^{-\eta}$.  Thus, we have a contradiction whenever $\eps < |X|^{-\eta}$.
\end{proof}

%% file: experiments.tex
We conducted experiments on real data from a Brazilian social
network called Apontador that provides location-based recommendations and
reviews.  The data set was introduced in the context of spam fighting in a
recent work by Costa et~al.~\cite{CostaMBB14} and is available from the
authors upon request. The
data set consists of $7076$ instances of so-called ``tips'' half of which are labeled as
``spam''.  Tips are pieces of user-provided content associated with the places
listed on Apontador.  The paper distinguishes between different types of spam, but 
the distinction does not matter for us, so
we
will only consider one category. There are $60$ features in total, but to
facilitate the modeling of a cost function we restricted our attention to the
$15$ most discriminative features as indicated by previous
work~\cite{CostaMBB14}. We normalized
all features of the data to have zero mean and unit standard deviation.

The goal of our cost function is not primarily to capture monetary cost of changing
certain attributes. Apart from attributes like ``number of followers'',
most attributes are technically easy to change. Rather the goal of a cost function is to
capture the loss in expected revenue that a spammer experiences when changing
certain parts of the spam message. If, for instance, it is essential for the
spam message to contain a URL or contact information, then the spammer
experiences lost revenue when such information is omitted. Similarly, the
spammer could choose to post his messages on the pages of lower-rated places, but 
such pages are less frequented and hence his utility decreases. Similar reasoning
applies to the modeling of the other attributes. Cheap attributes are those
that can be changed without a loss in utility for the spammer. For example,
the ``number of words'' is not robust as the spammer can freely choose to write
longer or shorter messages.

With this intuition in mind, we model our cost function as a simple linear
function truncated at $0$ to make it non-negative. That is we consider a cost
function of the form $c(x,y)=\langle\alpha,y-x\rangle_+.$ Truncation at $0$ is
a meaningful modeling decision, since a spammer doesn't derive any utility
from, say, decreasing the number of his followers even though it is costly to
increase this attribute. 

The cost vector $\alpha$ specifies for each attribute
a coefficient quantifying the cost of changing that attribute.
We do not attempt to construct as realistic a cost function as possible.
We only distinguish between three types of cost: somewhat costly to increase
(coefficient $1$), somewhat costly to decrease (coefficient $-1.0$) and cheap to
increase (coefficient $0.1$). The concrete values of these coefficients are rather
arbitrary and different choices may be more suitable.  
The next table details each feature with its description and its associated cost. 
For a more detailed explanation of these features, the reader is referred
to~\cite{CostaMBB14}.

\begin{center}
\begin{tabular}{|l|l|l|}
\hline
& Description & Cost coefficient\\
\hline
 1 & Number of tips on the place & $-1$\\
 2 & Place rating & $-1$ \\
 3 & Number of emails  & $-1$\\
 4 & Number of contact information  & $-1$\\
 5&  Number of URLs  & $-1$\\
 6&  Number of phone numbers &  $-1$\\
 7&  Number of numeric characters&  $-1$\\
 8&  SentiStrength score& $1$ \\
 9&  Combined-method& $1$\\
10&  Number of words& $0.1$\\
11&  Ratio of followers to followees & $1$ \\
12&  Number of distinct $1$-grams & $0.1$ \\
13&  Number of tips posted by user & $0.1$\\
14&  Number of followers & $1$\\
15&  Number of capital letters & $0.1$\\
\hline
\end{tabular}
\end{center}

We made no attempt to arrive at a perfectly-realistic cost function. Instead
our focus is on a qualitative comparison of our approach with a standard SVM
classifier, which does not take gaming into account. We selected SVM as a
representative classifier as it was shown in previous work~\cite{CostaMBB14}
to achieve high classification accuracy on this data set compared with other
standard classifiers. For simplicity and increased interpretability, we use a
\emph{linear} SVM which still achieves high accuracy.

If we were to assume that our model of gaming and choice of cost function were
perfectly correct, then a standard SVM would perform very poorly when compared
with our algorithm. To obtain a more balanced comparison, we take modeling
inaccuracies into account in our experiments. Specifically, we account for two
potential inaccuracies in our model:
\begin{enumerate}
\item The true cost function is not the one on which we train our algorithm.
\item The amount of gaming varies and does not necessarily 
correspond to the threshold predicted by our theoretical framework.
\end{enumerate}
Finally, we explore a convenient way to interpolate between the classifier
suggested by our approach and standard classifiers. This leads to different
trade-offs which are more favorable in certain settings.

\subsection{Comparison with SVM under robustness to modeling errors}
We now show that our method is robust to 
significant modeling errors while simultaneously outperforming SVM even if
only a small amount of gaming occurs.

To formalize our error model, we assume that there is a true 
underlying cost function which differs from the cost function we feed into
Algorithm~\ref{alg:linear}. We imagine that the true cost function is 
some mixture of the linear cost function described above, plus a squared
Euclidean distance term:
\begin{equation}\label{eq:true}
 c_{\mathrm{true}}(x,y) = (1 - \eps)\ip{ \alpha }{ y - x }_+ + \eps \twonorm{ x - y }^2. 
\end{equation}
On the other hand, we run our algorithm on a cost function which is incorrect
in two ways.  First, it is separable, so it necessarily ignores the
squared-distance  term.  Second, we do not imagine that we have correctly identified $\alpha$, and we replace it with some $\alpha'$:
\[ c_{\mathrm{assumed}}(x,y) = \ip{ \alpha' }{y - x }_+. \]
The addition of the Euclidean norm in \eqref{eq:true} reflects the possibility that our separability assumption does not exactly hold.  The difference between $\alpha$ and $\alpha'$ reflects the possibility that we may not even have accurately identified the separable part.
We stress that not only does our algorithm not know the true cost function, it also does not know the parameter $\eps$, or how much $\alpha$ differs from $\alpha'$.

For our experiments, we considered a range of values of $\eps$, and we generated $\alpha'$ from $\alpha$ at random by adding Gaussian noise and re-normalizing.
We develop our classifier using $c_{\mathrm{assumed}}$, but then for tests allow \cnt to best-respond to the 
classifier given the cost function $c_{\mathrm{true}}$.
We note that finding the best response to a linear classifier given the cost function $c_{\mathrm{true}}$ is a simple calculus problem.

The other parameter we varied is the amount of gaming allowed. In our
theoretical framework above, the \cnt is always willing to pay a cost of up to
$2$, since his payoff for switching is $1 - (-1) = 2$.  To relax this
assumption and vary the amount of
gaming allowed, we multiply both $c_{\mathrm{true}}$ and $c_{\mathrm{assumed}}$ by $2/t$; we say
that this allows $t$ units of gaming.  Notice that by the definition of
$c_{\mathrm{true}}$, this means that the \cnt is willing to move distance $t$ in
the direction of $\alpha$, and possibly more in other directions.  As
mentioned above, we have normalized the standard deviation of all attributes
to be~$1$. 

Within the above error model, we compare our algorithm with SVM as a 
representative standard classifier. 
Figures \ref{fig:compare}  show that our algorithm outperforms SVM, even under
a small amount of gaming, and even in the presence of significant modeling
errors.

\begin{figure}[ht!]
\includegraphics[width=0.49\textwidth]{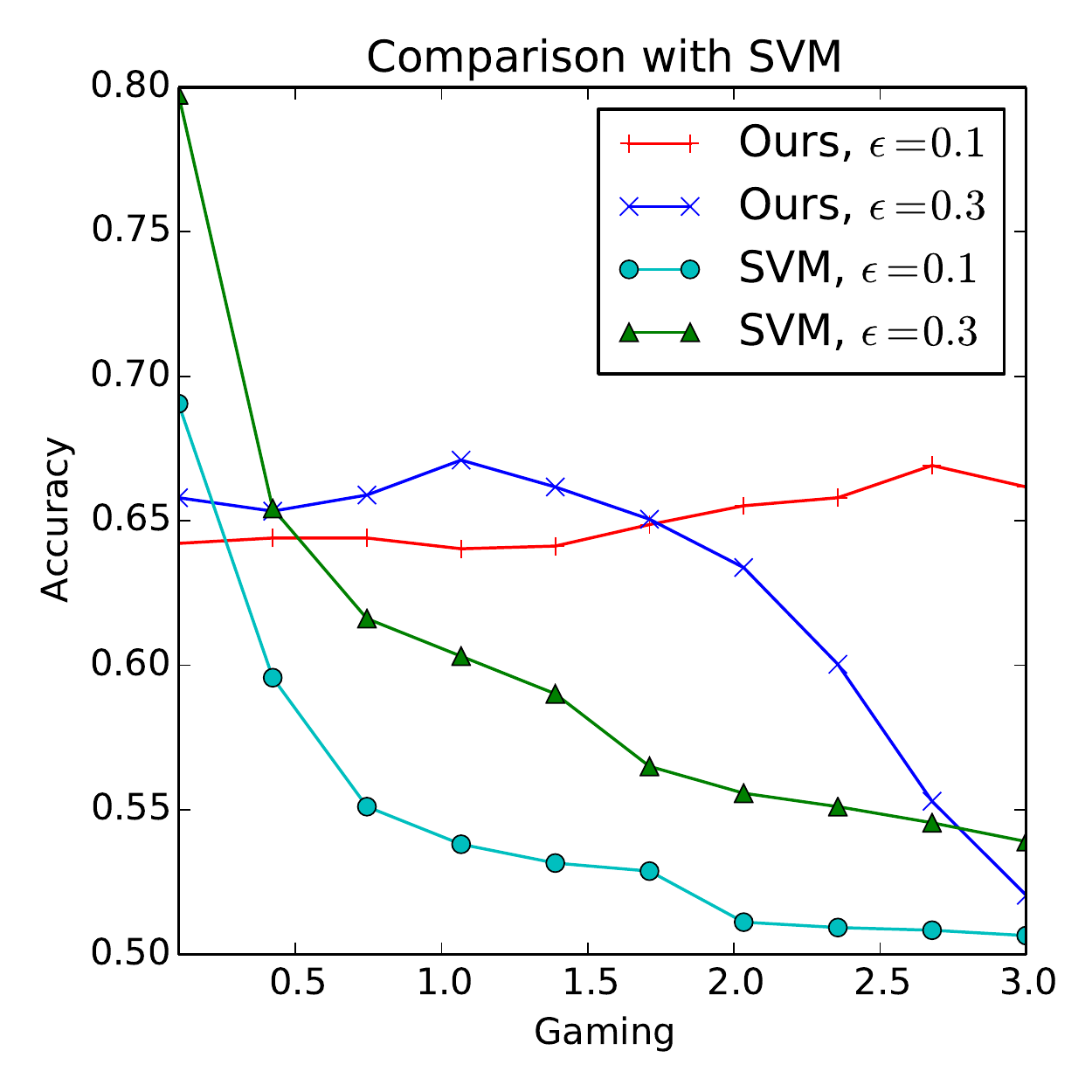}
\includegraphics[width=0.49\textwidth]{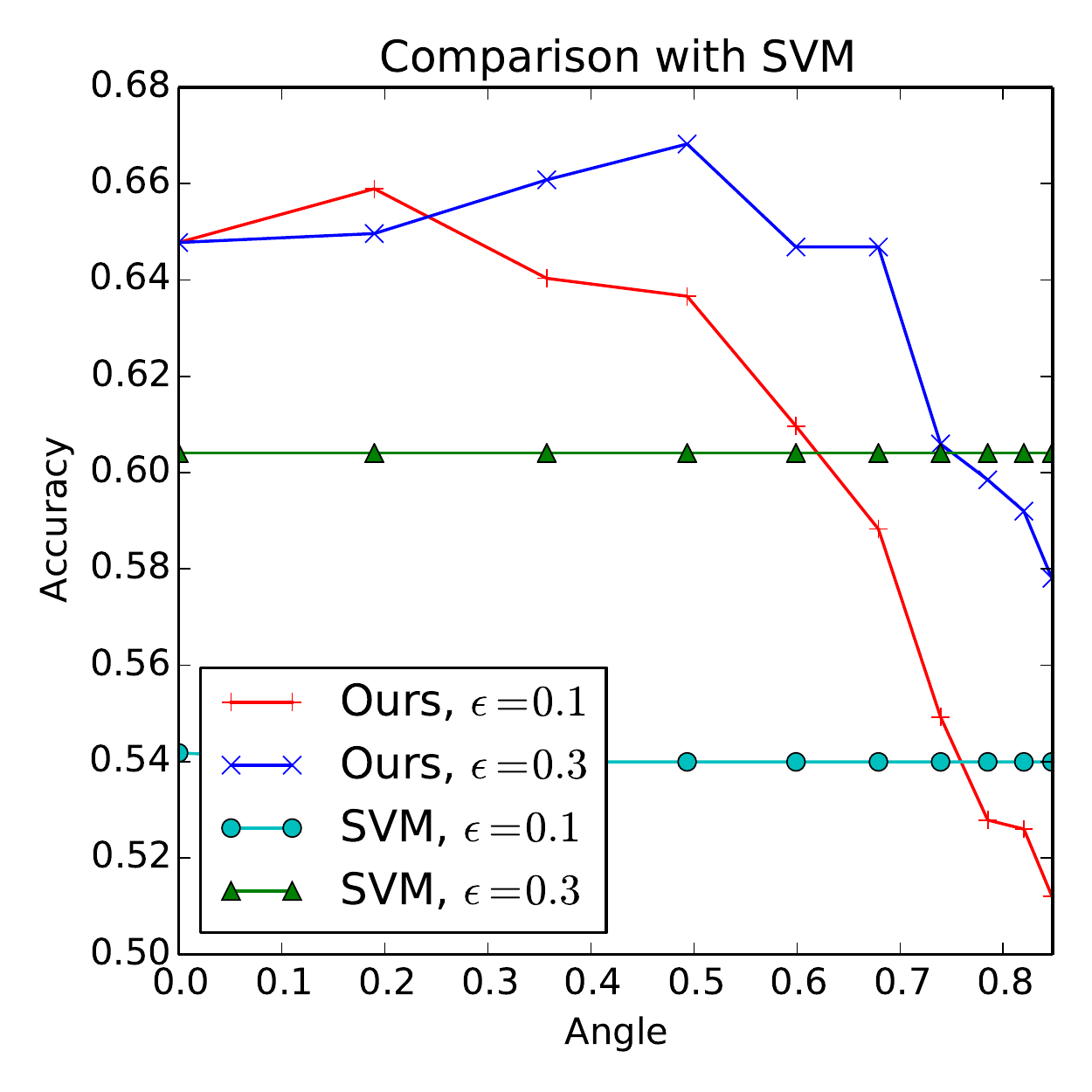}
\caption{{\bf Left:} Our algorithm compared with SVM as the amount of gaming is increased.  
The $x$-axis tracks the amount of gaming, which is quantified as described above. 
The parameter $\eps$ in $c_{\mathrm{true}}$ is specified in the legend.
We have set $\sin \theta( \alpha, \alpha' ) = 0.394$ (again, $\alpha'$ was randomly generated from $\alpha$ by adding Gaussian noise and re-normalizing).
{\bf Right:}
Our algorithm compared with SVM as the angle between $\alpha$ and $\alpha'$
increases. The $x$-axis measures the angle $\sin\theta(\alpha,\alpha')$. 
The amount of gaming was fixed at $1.0$, and the parameter $\eps$ in $c_{\mathrm{true}}$ is 
specified in the legend.
}\label{fig:compare}
\end{figure}

\subsection{A hybrid approach for higher accuracy}
In practice it is convenient to start from a standard classifier and make it
more robust to gaming and as opposed to adopting an entirely new classifier.
Our framework gives a convenient way to incorporate a set of known classifiers
into the design of a strategy-robust classifier. As we show below this can
lead to more favorable trade-offs between gaming and accuracy.

The basic idea is to use each known classifier as a feature to which we assign a positive
weight in the cost function. In other words, we stipulate that
the classifier is by itself a somewhat reliable attribute of the data. Below
we try out this hybrid approach by combining our classifier with the standard SVM classifier.  
Indeed, we find in our experiments that the hybrid has higher accuracy in a
robust range of parameters. This is shown in Figure~\ref{fig:interpolate}.

In the case of a linear SVM, the decision boundary is given by a vector
$\beta$ and we can simple add this vector to our cost function.
We assume that the true cost function $c_{\mathrm{true}}$ is as above, 
but we modify $c_{\mathrm{assumed}}$ as:
\[ c_{\mathrm{assumed}}(x,y) = \ip{ (1 - \gamma)\alpha' + \gamma \beta }{y - x}_+, \]
where $\beta$ are the SVM coefficients learned from the training data set.

\begin{figure}
\begin{center}
\includegraphics[width=0.49\textwidth]{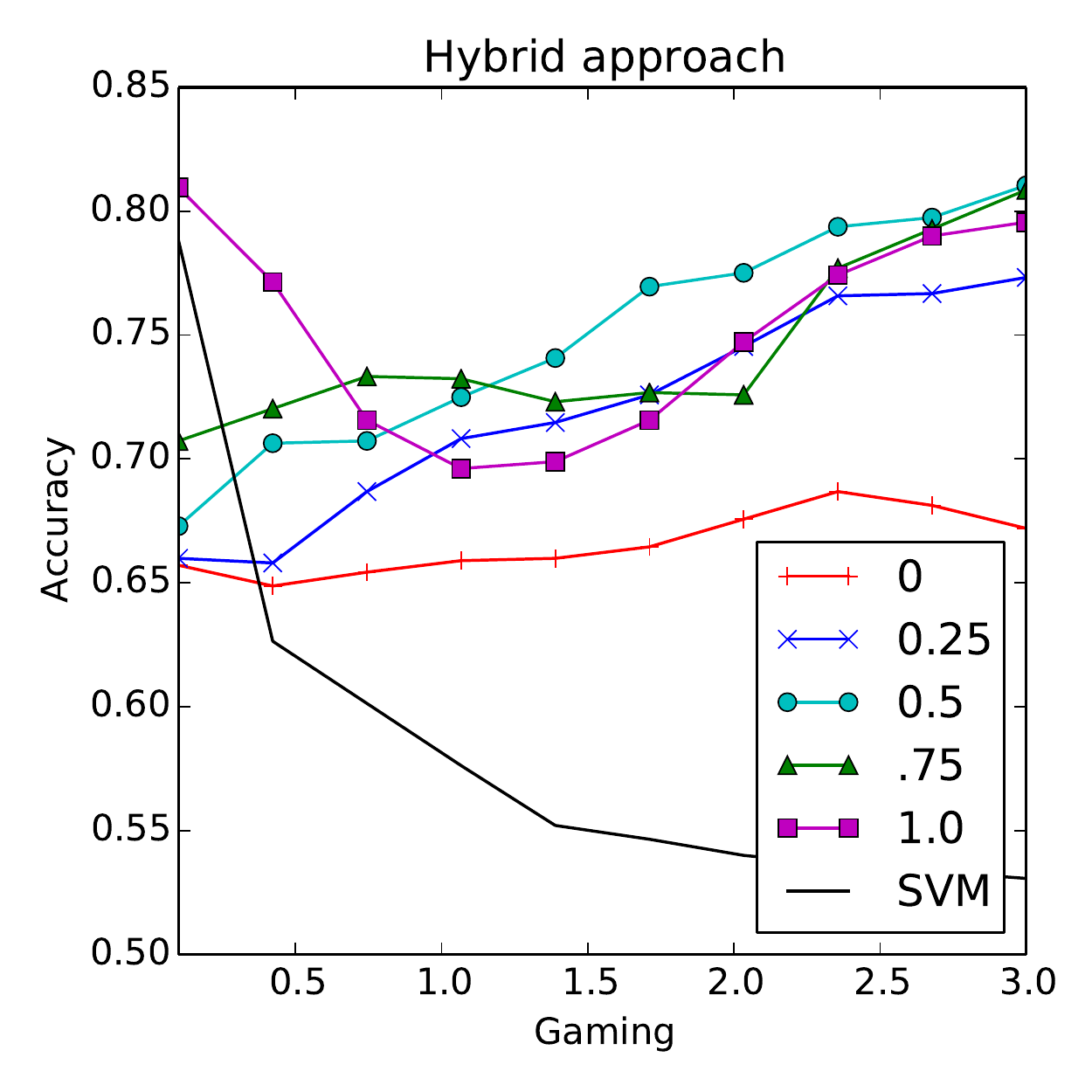}
\end{center}
\caption{Interpolation between the SVM classifier and our classifier. 
In the model above, we begin with a cost function that has $\sin\theta( \alpha, \alpha' ) \approx 0.2$ and $\eps = 0.2$.  
Then we mix the weight vector $\alpha'$ with the weights $\beta$ obtained from
SVM to arrive at $ \alpha'' = (1 - \gamma) \alpha + \gamma \beta$ which
defines the assumed cost function.
The lines in the plot above show what happens as the amount of gaming
increases when setting $\gamma=0,0.25,0.5,0.75,1$.
Notice that the difference between the SVM curve and the curve with $\gamma=1$ is that
the classifier for $\gamma=1$ is shifted according to our algorithm.
}\label{fig:interpolate}
\end{figure}

%% file: minproof.tex
\begin{proof}[Proof of Theorem \ref{lem:minsep}]
Fix $h \in \cC$.  As in the proof of Theorem \ref{thm:separable}, we begin by defining the set $\Gamma(f)$ of $x \in X$ so that $f(\Delta(x)) = 1$ when $\Delta$ is a best response to $f$.
For every $f \in \cC$, we have
\begin{align*}
 \Gamma(f) 
 &:= \inset{ x \suchthat \max \{ f(y) \,:\, y \in \Gamma(x) \} = 1 } \\ 
 &=  \inset{ x \suchthat (\exists y \in X,\ \exists b \in \mathcal{B})(f(y) = 1,\ b(x,y) < 2) } \\ 
 &=  \inset{ x \suchthat (\exists b \in \mathcal{B})(b_1(x) > \min \{b_2(y)\,:\, f(y) = 1\} - 2)  }\\
 &= \cup_{b \in \mathcal{B}} \big\{ x \suchthat  b_1(x) >  \min\{ b_2(y)\,:\, f(y)=1  \} - 2 \big\}~ . 
 \end{align*}
 Now we can again restrict our attention to nicely structured functions. For any $f$, let
 \begin{equation}\label{eq:specialf}
  f'(y) = \begin{cases}
    ~~  1 & \text{if} ~~~(\forall b \in \mathcal{B})( b_2(y) \geq \min\{ b_2(z)\,:\, f(z) = 1  \})  \\ 
     -1 & \text{otherwise~.} \end{cases}
  \end{equation}
 Then, as in the proof of Theorem~\ref{thm:separable}, we have 
 \[ \min \{ b_2(y)  \,:\, f(y) = 1\} ~=~ \min \{ b_2(y)\,:\, f'(y)=1 \} \]
 for all $b \in \mathcal{B}$.
Indeed, 
\[ \min\{ b_2(y) \,:\, f'(y)=1 \} ~\geq~  \min \{  b_2(z)\,:\, f(z) = 1\} \]
by definition of $f'$,
and
\begin{align*}
  \min \{ b_2(y)\,:\, & f'(y)=1 \} \\
 =&\ \min \inset{ b_2(y) \suchthat y \in \cap_{b \in \mathcal{B}} \inset{ w \suchthat b_2(w) 
         \geq \min \{  b_2(z)\,:\, f(z) = 1   \} } }   \\
\leq &\ \min \inset{ b_2(y) \suchthat f(y) = 1} ~, 
\end{align*}
using the fact that 
\[ \inset{y: f(y) = 1} ~\subset ~\inset{ w \suchthat b_2(w) 
\geq \min \{ b_2(z)\,:\, f(z)=1 \} }
\]
for all $b \in \mathcal{B}$.
Thus,
\begin{align*}
 \Gamma(f) ~=&\ \cup_{b \in \mathcal{B}} \big\{\, x \suchthat b_1(x) > 
   \min \{ b_2(y)   \,:\, f(y)=1 \} -2 \, \big\} \\
 =&\ \cup_{b \in \mathcal{B}} \big\{\, x \suchthat b_1(x) > 
   \min \{  b_2(y) \,:\, f'(y) = 1 \} -2   \,\big\}~ =~ \Gamma(f')
 ~.
 \end{align*}
Thus, as before, the payoff to \jr if she plays $f$ is the same as if she plays $f'$:
 \begin{align*}
  {\mathbb{P}}  \left( h(x) =  \right.  \max & \{ f(y)\,:\, {y \in \Gamma(x)} \left.  \right)  \\
  = &\  {\mathbb{P}}\left( \inparen{ \Gamma(f) \triangle h }^c \right)  \\
  = &\  {\mathbb{P}}\left( \inparen{ \Gamma(f') \triangle h}^c \right)  \\ 
  = &\  {\mathbb{P}}\left( h(x) = \max\{f'(y)\,:\, y \in \Gamma(x)\}  \right)   
   ~. 
   \end{align*}
 Thus, it suffices to consider classifiers $f$ of the form \eqref{eq:specialf}.  Moving the quantifiers around, it suffices to consider classifiers of the form
\begin{equation}\label{eq:specialfagain}
  f = \min  \{ \thresh{b_2}{\vec{s}_b} \,:\, b \in \mathcal{B} \}
\end{equation}
 for 
\[  S_{\mathcal{B}} := \vec{s} \in \bigoplus_{b \in \mathcal{B}} \inparen{ b_2(X) \cup \inset{\infty}} ~.  \] 
Here, $\vec{s}_b$ plays the role of 
$\min \{b_2(z)\,:\, f(z) = 1 \}$, 
and $\vec{s}_b = \infty$ means that $f(z) = -1$ for all $z \in X$. 
 For $f$ as in \eqref{eq:specialfagain}, we have
 \[ \Gamma(\min \{  \thresh{b_2}{\vec{s}_b}\,:\, b \in \mathcal{B} \}  ) 
 = \bigcup_{b \in \mathcal{B}} \inset{ x \suchthat b_1(x) > \vec{s}_b - 2}, \]
 and a best-possible payoff to \jr is obtained by finding the best thresholds $\vec{s}$: 
\begin{align*}
  OPT_h(\cD, c) 
  = &\  1 - \inf_{ \vec{s} \in S_{\mathcal{B}} } 
    \big\{ {\mathbb{P}} \left( h(x) \neq 
      \min\{ \thresh{b_1}{\vec{s}_b- 2} (x)\,:\, {b \in \mathcal{B}} \}        \right) \big\} \\
  =: &\ 1 - \inf_{\vec{s} \in S_{\mathcal{B}}}\{ \err(\vec{s}) \}~. 
  \end{align*}
In Algorithm~\ref{alg:minsep}, \jr returns 
\[ f = \min\{ \thresh{b_2}{\vec{s}^*_b   } \,:\, b \in \mathcal{B} \} ~,\]
and as above
the payoff to \jr from this $f$ is
\[ {\mathbb{P}} \left( h(x) \neq \thresh{c_1}{\vec{s}^* - 2}(x) \right) = 1 - \err(\vec{s}^*)~. \]
As in the proof of Theorem~\ref{thm:separable}, to prove Theorem~\ref{lem:minsep} it suffices to show that for all $h \in \cC$,
\begin{equation}\label{eq:want2}
 \err(\vec{s}^*) ~\leq~ \inf \{ \err(\vec{s}) \,:\, \vec{s} \in S_{\mathcal{B}}  \} + \eps ~. 
\end{equation}
As before, we have
\begin{equation}\label{claim:12}
\widehat{\err}(\vec{s}^*) ~=~
   \min\{ \widehat{\err}(\vec{s}) \,:\,  \vec{s} \in S_{\mathcal{B}}  \} ~,
\end{equation}
so it suffices to establish that
 $\widehat {\err}(\vec{s})$ is close to $\err(\vec{s})$, uniformly over $s \in S_{\mathcal{B}}$.
  \begin{claim}\label{claim:22}
  With probability at least $1 - \delta$,  for all $h \in \cC$ and $\vec{s} \in S_{\mathcal{B}}$,
 \[ \inabs{ \widehat{\err}(\vec{s}) - \err(\vec{s}) } \leq 4R_m(\cC) 
 + 8 \sqrt{\mbox{$ \frac{|\mathcal{B}|\ln(m+1)}{m} $}} 
 +   \sqrt{\mbox{$ \frac{ 2 \ln(2/\delta)}{m} $} }~. \]
In particular, if the hypotheses of the lemma are met, with probability at least $1 - \delta$,
\[ \sup \big\{ \left| \widehat{\err}(\vec{s}) - \err(\vec{s})\right| \,:\, h \in \cC, \vec{s} \in S_{\mathcal{B}} \big\}
~\leq~ \eps/2~. \]
 \end{claim}
 \begin{proof}
 Again, this follows very similarly to the proof of Theorem \ref{thm:separable}.  
 We need to bound the absolute value of the following difference:
 \begin{align*}
  \oneoverm \sum_{j=1}^m &\ind{ h(x_j) \neq 
        \min \{\thresh{b_1}{\vec{s}_b-2} (x_j)\,:\, b \in \mathcal{B} \}   }  \\
   -&
    \EE_{x \sim \cD}\big[ \ind{ h(x) \neq 
        \min \{ \thresh{b_1}{\vec{s}_b - 2}(x) \,:\,  b \in \mathcal{B} \}    } \big]
  \end{align*}
for all $h \in \cC$ and $\vec{s} \in S_\mathcal{B}$.
 As before (via, say, Theorem 3.2 in~\cite{BBL05}), we have for all $h \in \cC, \vec{s} \in S_{\mathcal{B}}$,
 \begin{equation}\label{eq:std2}
  \inabs{ \frac{1}{m} \sum_{j=1}^m \ind{ h(x_j) \neq \min_{b \in \mathcal{B}} \thresh{b_1}{\vec{t}_b}(x_j) } - \EE_{x \sim \cD} \ind{h(x) \neq \min_{b \in \mathcal{B}} \thresh{b_1}{\vec{t}_b}(x) }} \leq 2 R_m(\mathcal{X}) + \sqrt{ \frac{ 2\ln(2/\delta) }{m} }, 
  \end{equation}
 where 
 \[ \mathcal{X} = \inset{ h \cdot \min_{b \in \mathcal{B}} \thresh{b_1}{\vec{s}_b-2} \suchthat h \in \cC, \vec{s} \in S_{\mathcal{B}} }.\]
As before,
\begin{equation}\label{eq:additive2}
R_m(\mathcal{X}) \leq 2 \inparen{ R_m( \cC) + R_m(\mathcal{H})}, 
\end{equation}
 where $\mathcal{H} = \inset{ \min_{b \in \mathcal{B}} \thresh{b_1}{\vec{s}_b - 2} \suchthat \vec{s} \in S_{\mathcal{B}}}$, so it remains to bound $R_m(\mathcal{H})$.
  For fixed $x_1,\ldots, x_m \in X$, we have
 \begin{align*}
 \EE_{\sigma} & \bigg[ \sup \bigg\{ \oneoverm \sum_{i=1}^m \sigma_i \min_{b \in \mathcal{B}} \thresh{b_1}{\vec{s}_b - 2}(x_i) \,:\, \vec{s} \in S_{\mathcal{B}} \bigg\}  \bigg] \\
 &= \EE_\sigma \bigg[ \sup \bigg\{ \oneoverm \sum_{i=1}^m \sigma_i \min_{b \in \mathcal{B}} \thresh{b_1}{\vec{s}_b-2}(x_i) \,:\, \vec{s}  \in \bigoplus_{b \in \mathcal{B}} \inset{ s_{j,b} : j \in [m+1] } \bigg\} 
  \bigg] \\
 &\leq 2 \sqrt{ \mbox{$\frac{\ln((m+1)^{|\mathcal{B}|)}}{m} $}}~.
 \end{align*}  
 Thus, we have
 \[ R_m(\mathcal{H}) \leq 2\sqrt{\frac{|\mathcal{B}|\ln(m+1)}{m}}, \]
 and, along with Equation \ref{eq:std2}, this finishes the claim.
 \end{proof}
 As in the proof of Theorem \ref{thm:separable}, Equation \ref{claim:12} and Claim \ref{claim:22} finish the proof of Theorem \ref{lem:minsep}.
\end{proof}

%% file: gaming.bbl
\newcommand{\etalchar}[1]{$^{#1}$}
\begin{thebibliography}{CdCMBB14}

\bibitem[BBL05]{BBL05}
St{\'e}phane Boucheron, Olivier Bousquet, and G{\'a}bor Lugosi.
\newblock Theory of classification: A survey of some recent advances.
\newblock {\em ESAIM: probability and statistics}, 9:323--375, 2005.

\bibitem[BKS12]{BKS12}
Michael Br{\"{u}}ckner, Christian Kanzow, and Tobias Scheffer.
\newblock Static prediction games for adversarial learning problems.
\newblock {\em Journal of Machine Learning Research}, 13:2617--2654, 2012.

\bibitem[BS09]{BS09}
Michael Br{\"{u}}ckner and Tobias Scheffer.
\newblock Nash equilibria of static prediction games.
\newblock In {\em Proc.~$23$rd NIPS 2009}, pages 171--179, 2009.

\bibitem[BS11]{BS11}
Michael Br{\"{u}}ckner and Tobias Scheffer.
\newblock Stackelberg games for adversarial prediction problems.
\newblock In {\em Proc~$17$th {ACM} {SIGKDD}}, pages 547--555, 2011.

\bibitem[CdCMBB14]{CostaMBB14}
Helen Costa, Luiz~Henrique de~Campos~Merschmann, Fabr{\'{\i}}cio Barth, and
  Fabr{\'{\i}}cio Benevenuto.
\newblock Pollution, bad-mouthing, and local marketing: The underground of
  location-based social networks.
\newblock {\em Inf. Sci.}, 279:123--137, 2014.

\bibitem[CP14]{CitronP14}
Danielle~Keats Citron and Frank Pasquale.
\newblock The scored society: Due process for automated predictions.
\newblock {\em 89 Washington Law Review}, 1, 2014.

\bibitem[DDM{\etalchar{+}}04]{Dalvi04}
Nilesh~N. Dalvi, Pedro Domingos, Mausam, Sumit~K. Sanghai, and Deepak Verma.
\newblock Adversarial classification.
\newblock In {\em Proc~$10$th {ACM} {SIGKDD}}, pages 99--108, 2004.

\bibitem[EKST10]{Evans}
M.D.R. Evans, J.~Kelley, J.~Sikora, and D.~J. Treiman.
\newblock Family scholarly culture and educational success: Evidence from 27
  nations.
\newblock {\em Research in Social Stratification and Mobility}, 28:171--197,
  2010.

\bibitem[GSBS13]{GSBS13}
Michael Gro{\ss}hans, Christoph Sawade, Michael Br{\"{u}}ckner, and Tobias
  Scheffer.
\newblock Bayesian games for adversarial regression problems.
\newblock In {\em Proc.~$30$th ICML}, pages 55--63, 2013.

\bibitem[KCP10]{KCP10}
Dmytro Korzhyk, Vincent Conitzer, and Ronald Parr.
\newblock Complexity of computing optimal {Stackelberg} strategies in security
  resource allocation games.
\newblock In {\em Proc.~AAAI}, 2010.

\bibitem[KYK{\etalchar{+}}11]{KYKCT11}
Dmytro Korzhyk, Zhengyu Yin, Christopher Kiekintveld, Vincent Conitzer, and
  Milind Tambe.
\newblock Stackelberg vs. {Nash} in security games: An extended investigation
  of interchangeability, equivalence, and uniqueness.
\newblock {\em J. Artif. Intell. Res.(JAIR)}, 41:297--327, 2011.

\bibitem[LT91]{LT91}
Michel Ledoux and Michel Talagrand.
\newblock {\em Probability in Banach Spaces: isoperimetry and processes},
  volume~23.
\newblock Springer, 1991.

\bibitem[Val84]{Valiant84}
Leslie Valiant.
\newblock A theory of the learnable.
\newblock {\em Communications of the ACM}, 27(11):1134--1142, 1984.

\end{thebibliography}
